\documentclass{article}

\usepackage{microtype}
\usepackage{graphicx}
\usepackage{subfigure}
\usepackage{booktabs} 

\usepackage{subcaption}

\usepackage{hyperref}

\usepackage[accepted]{icml2024}

\usepackage{amsmath}
\usepackage{amssymb}
\usepackage{mathtools}
\usepackage{amsthm}

\usepackage{thmtools, thm-restate} 
\usepackage{nicefrac} 
\usepackage{xfrac}
\allowdisplaybreaks[1] 

\usepackage{algorithm} 
\usepackage{algorithmicx}
\usepackage[noend]{algpseudocode}
\algnewcommand\algorithmicforeach{\textbf{for each}} 
\algdef{S}[FOR]{ForEach}[1]{\algorithmicforeach\ #1\ \algorithmicdo} 
\MakeRobust{\Call} 

\DeclareMathOperator*{\argmin}{\arg\!\min}

\newcommand{\E}{\mathbb{E}}

\newcommand{\learnmodel}{\operatorname{\mathsf{LearnModel}}}
\newcommand{\LM}{\mathsf{LM}}

\newcommand{\estimatePR}{\operatorname{\mathsf{Estimate}\PR}} 
\newcommand{\minimizePR}{\operatorname{\mathsf{Minimize}\PR}} 
\newcommand{\fhat}{\hat{f}} 
\newcommand{\ER}{\mathsf{ER}} 
\newcommand{\PR}{\mathsf{PR}} 
\newcommand{\PRdagger}{\PR^\dagger} 
\newcommand{\PRtilde}{\widetilde{\PR}} 
\newcommand{\PRhat}{\widehat{\PR}} 
\newcommand{\calA}{\mathcal{A}} 
\newcommand{\calD}{\mathcal{D}} 
\newcommand{\calR}{\mathcal{R}} 

\newcommand{\R}{\mathbb{R}} 
\newcommand{\Z}{\mathbb{Z}} 
\renewcommand{\S}{\mathbb{S}} 

\newcommand{\KL}{\mathsf{KL}}
\newcommand{\TV}{\mathsf{TV}}

\newcommand{\estimatekl}{\operatorname{\mathsf{EstimateKL}}} 

\newtheorem{theorem}{Theorem}
\newtheorem*{theorem*}{Theorem}

\newtheorem*{lemma*}{Lemma}
\newtheorem{definition}{Definition}
\newtheorem*{definition*}{Definition}
\newtheorem{proposition}{Proposition}
\newtheorem*{proposition*}{Proposition}
\newtheorem{claim}{Claim}
\newtheorem*{claim*}{Claim}
\newtheorem{example}{Example}
\newtheorem*{example*}{Example}

\newtheorem*{corollary*}{Corollary}
\newtheorem{assumption}{Assumption}
\newtheorem*{assumption*}{Assumption}

\usepackage[capitalize]{cleveref}
\crefname{claim}{claim}{claims}
\crefname{assumption}{assumption}{assumptions}

\usepackage{tikz}
\usetikzlibrary{arrows}

\usepackage{graphicx}
\graphicspath{ {images/} }
\usepackage{wrapfig}
\usepackage[font={small,it}]{caption} 
\usepackage{subcaption}
\usepackage{float}
\usepackage{comment}

\usepackage{standalone} 

\ifodd 1

\newcommand{\wt}[1]{{\color{blue}(WT: #1)}}

\else
\newcommand{\wt}[1]{}
\fi

\newcommand{\OPT}{\texttt{OPT}}

\newcommand{\modelDimen}{d_\Theta}
\newcommand{\modelDiameter}{D_\Theta}
\newcommand{\distributionDimen}{d_\Phi}
\newcommand{\distributionDiameter}{D_\Phi}

\newcommand{\totaltime}{T_{\textsf{total}}}

\newcommand{\distParaMap}{\varphi}

\newcommand{\squishlist}{
\begin{list}{{{\small{$\bullet$}}}}
{\setlength{\itemsep}{3pt}      \setlength{\parsep}{1pt}
\setlength{\topsep}{1pt}       \setlength{\partopsep}{0pt}
\setlength{\leftmargin}{1em} \setlength{\labelwidth}{1em}
\setlength{\labelsep}{0.5em} } }
\newcommand{\squishend}{  \end{list}}

\newcommand{\cc}[1]{\ensuremath{\mathsf{#1}}} 

\newcommand{\xhdr}[1]{\vspace{-5pt}\paragraph*{
\bf {#1.}}}

\usepackage{booktabs}

\usepackage{thmtools}
\usepackage{thm-restate}

\theoremstyle{remark}
\newtheorem{remark}[theorem]{Remark}

\usepackage[textsize=tiny]{todonotes}

\usepackage{changepage}

\icmltitlerunning{Performative Prediction with Bandit Feedback:
Learning through Reparameterization}

\begin{document}

\twocolumn[
\icmltitle{Performative Prediction with Bandit Feedback:\\
Learning through Reparameterization}



\icmlsetsymbol{equal}{*}

\begin{icmlauthorlist}
\icmlauthor{Yatong Chen}{ucsc}
\icmlauthor{Wei Tang}{columbia,cuhk}
\icmlauthor{Chien-Ju Ho}{wustl}
\icmlauthor{Yang Liu}{ucsc}
\end{icmlauthorlist}

\icmlaffiliation{ucsc}{Department of Computer Science and Engineering, University of California, Santa Cruz, California, United States.}
\icmlaffiliation{columbia}{Data Science Institute, Columbia University}
\icmlaffiliation{cuhk}{Department of Decisions, Operations, and Technology, the Chinese University of Hong Kong}
\icmlaffiliation{wustl}{Department of Computer Science and Engineering, Washington University in St. Louis}

\icmlcorrespondingauthor{Yang Liu}{yangliu@ucsc.edu}

\icmlkeywords{performative prediction, zeroth-order optimization, reparametrization, non-convex}

\vskip 0.3in
]

\printAffiliationsAndNotice{} 

\begin{abstract}
Performative prediction, as introduced by \citeauthor{perdomo2020performative}, is a framework for studying social prediction in which the data distribution itself changes in response to the deployment of a model. 
Existing work in this field usually hinges on three assumptions that are easily violated in practice: that the performative risk is convex over the deployed model, that the mapping from the model to the data distribution is known to the model designer in advance, and the first-order information of the performative risk is available. 
In this paper, we initiate the study of performative prediction problems that do not require these assumptions.
Specifically, we develop a {\em reparameterization} framework that reparametrizes the performative prediction objective as a function of the induced data distribution. We then develop a two-level zeroth-order optimization procedure, where the first level performs iterative optimization on the distribution parameter space, and the second level learns the model that induces a particular target distribution 
at each iteration. 
Under mild conditions, this reparameterization allows us to transform the non-convex objective into a convex one and achieve provable regret guarantees. 
In particular, we provide a regret bound that is sublinear in the total number of performative samples taken and is only polynomial in the dimension of the model parameter.

\end{abstract}

\section{Introduction}
\label{sec:introduction}
\emph{Performative prediction}, as introduced by \citeauthor{perdomo2020performative}, provides a framework for studying prediction and risk minimization when the data distribution itself changes in response to the deployment of a model. Such phenomena, usually referred to as "performativity," are prevalent in various social prediction contexts, including education, recommendation systems, and criminal prediction, among others \cite{perdomo2020performative, chen2023model, hardt2016strategic, dong2018strategic, kleinberg20induce}. For instance, consider a college admission process that places significant importance on standardized test scores. This process can incentivize students to invest more effort in test preparation, ultimately leading to a pool of applicants with much higher test scores than initially expected. This phenomenon is also prevalent in real-world applications, particularly in large-scale online recommendation systems, where the high frequency of updates to the recommendation algorithm can reshape users' future behavior. For example, video platforms such as TikTok, Netflix, and YouTube provide personalized recommendations that can influence users' future preferences and lead to shifts in the user-advertiser interaction patterns, thereby creating a dynamic and evolving data distribution.

More formally, consider the standard empirical risk minimization (ERM) problem defined by a loss function $\ell$, a model parameter space $\Theta \subset \R^{d_\Theta}$ where $d_\Theta\in\Z_{>0}$, an instance space $Z = X \times Y$, and a \emph{fixed} data distribution $\calD$ over $Z$. The task is to find a model that minimizes the empirical risk defined as: $\ER(\theta, \calD):= \E_{z \sim \calD} [\ell(z;\theta)].$ Performative prediction extends this learning task by positing that the data distribution $\calD$ is \emph{not} fixed but is instead a function of the model parameter $\theta \in \Theta$. Here, we refer to $\calD(\cdot)$ as a \emph{distribution map}, and $\calD(\theta)$ as the data distribution \emph{induced} by the model $\theta$. The objective is then to minimize the \emph{performative risk}, defined as
\begin{align*}
\PR(\theta, \calD(\theta))
:= \E_{z \sim \calD(\theta)} [\ell(z;\theta)]~.
\end{align*}
Intuitively, the performative prediction risk evaluates the performance of the model $\theta$ on the resulting distribution $\calD(\theta)$ via the loss function $\ell$. When it is clear from the context, we also use $\PR(\theta)$ to shorthand the performative risk.


Optimizing the performative risk is generally challenging. In standard ERM, a convex loss function $\ell$ implies a convex empirical risk. However, as \citet{perdomo2020performative} observed, the performative risk $\PR(\cdot)$ may be non-convex even when the loss $\ell$ itself is convex. For this reason, earlier works \cite{perdomo2020performative, mendler2020stochastic, drusvyatskiy2020stochastic, brown2020performative} then focus on computing a \emph{performative stable} solution instead, which is easier to achieve using standard optimization tools like repeated risk minimization. A performative stable model is loss-minimizing \emph{on the data distribution it induces}, though other models may incur smaller losses on their respective induced distributions. However, as recent works \cite{miller2021outside, izzo2021learn} point out, such stable solutions may be highly suboptimal and, worse yet, may not exist in certain settings.

One major challenge in performative risk minimization is the unknown distribution map between the model parameter $\theta$ and the distribution $\calD(\theta)$ without making any structural assumption. 
For example, one can hardly anticipate the click-through rate of an ad without putting out the ad first. In the language of performative prediction, only by deploying a model $\theta$ can the learner observe data samples that are i.i.d realized from the induced data distribution $\calD(\theta)$. Due to this inherent uncertainty about $\calD(\theta)$, it is impossible to compute the gradient of $\PR(\theta)$ w.r.t $\theta$, not to mention finding a model with the lowest performative risk offline. Instead, the learner must interact with the environment and deploy models $\theta$ to explore the induced distributions $\calD(\theta)$, which involves deploying ``imperfect'' models on decision subjects.
    
In this paper, we propose to measure the loss incurred by deploying a sequence of models $\theta_1, \ldots, \theta_{\totaltime}$ by evaluating the following regret measured with respect to the total number of samples deployed during the process:
\begin{align}
\label{eqn:regret-def}
    \calR_N(\calA, \PR)
    = \sum_{\tau=1}^{\totaltime}  \sum_{i=1}^{n_{\tau}} \ell(z_{\tau}^{(i)}; \theta_{\tau})- N \cdot \PR(\theta_{\OPT})
\end{align}
where $N:=\sum_{\tau=1}^{\totaltime} n_{\tau}$ denotes the total number of observed data samples throughout the process, $\calA$ corresponds to the particular algorithm, and $\PR$ represents the objective function. This regret measures the suboptimality of the deployed sequence of models relative to a performative optimum $\theta_{\OPT} \in \operatorname{argmin}_\theta \mathrm{PR}(\theta)$ in terms of how much loss they incur on the population with $N$ decision subjects.

In contrast to earlier studies that primarily assess the final model's performance based on optimality rather than the cumulative loss incurred throughout the process, we argue that this constitutes a more practical evaluation metric in predictive scenarios involving multiple rounds of human feedback. In particular, since the process of finding the optimal performative model involves deploying sub-optimal models on human agents in the process, it is more appropriate to define regret on the total number of agents that are subjected to the ``imperfect’’ algorithmic system rather than only caring about whether the final model is optimal. 
We believe this provides a unique evaluation metric suitable for performative prediction.   

Later in \Cref{sec:putting-things-together}, we compare our proposed regret definition with the standard regret measured in $\totaltime$ in more detail and show that our algorithm is, in fact, also sublinear in the total deployment steps $\totaltime$. This, combined with the fact that sublinear regret implies model convergence (\Cref{proposition:sublinear-regret-implies-convergence}), also means that our algorithm can guarantee to output a model arbitrarily close to the performative optimal model $\theta_{\OPT}$. 
\vspace{-0.1in}
\subsection{Our Contributions}
Our main contributions are a two-level zeroth order optimization algorithm that achieves a sublinear regret bound measured using the total number of samples and a novel reparametrization framework attempting to tackle a particular non-convex performative prediction problem.
\vspace{-0.1in}
\xhdr{Reparametrization Framework}
Departing from previous work, we allow $\PR(\theta)$ to be non-convex in the model parameter $\theta$, but suppose it is convex in the \emph{data distribution} parameter $\phi \equiv \varphi(\theta)$. Informally, under mild conditions, we show that 
non-convex $\PR(\theta)$ can be reparameterized 
as a new (convex) function $\PRdagger(\phi)$ over the induced data distribution parameter $\phi$. 
We discuss detailed parametrization procedure in \Cref{subsec: overview}. 
\xhdr{Zeroth-Order Optimization Algorithm with Performativity}
Given the parametrization framework proposed above, we propose a two-level zeroth-order optimization procedure, which, to our knowledge, is novel in performative prediction.
We believe our method 
enjoys the following benefits:
\squishlist
    \item \textbf{No Requirement for Gradient Information} Unlike the traditional gradient-based optimization procedure, our method does not require the explicit calculation of \emph{gradients} that may be complex or unavailable.
    \item \textbf{Black-Box Models} Our method can still be effective when dealing with models or systems that are treated as black boxes, where the internal mechanisms are not well understood (such as complicated economic systems) since it doesn't require knowledge of the underlying model structure.
    \item \textbf{Robustness to Noise} In many real-world applications, objective function evaluations may be noisy or subject to uncertainty, such as modeling consumer behavior. Our method can handle noisy evaluations and make decisions that are robust to noise.
\squishend

Our main results can be summarized as follows:
\begin{restatable}[Informal]{thm}{}
There exists an algorithm that, under appropriate conditions, incurs regret
$\widetilde{O}((\modelDimen + \distributionDimen) \cdot N_{\KL}^{1/6} \cdot N^{5/6})$\footnote{$\widetilde{O}(\cdot)$ suppresses polylogarithmic factors in $N$ and the failure probability $1/p$.}
after $N$ performative samples\footnote{Samples that the learner deploys along the way of finding the performative optimal model.} with probability at least $1-p$, where $N_{\KL}$ depends on the sample efficiency of an off-the-shelf estimator for KL divergence, and $\modelDimen$ and $\distributionDimen$ denote the dimension of the model and distribution parameter space, respectively.
\end{restatable}
The $N_\KL$ term in our regret depends on the sample efficiency of the estimator for KL divergence. The detailed discussion is provided in \Cref{sec:learnmodel}. 

\subsection{Related Work}
Our work most closely relates to performative prediction and zeroth-order optimization. Due to page limit, we include additional related work in \Cref{sec:additional-related-work}, including detailed comparisons of our work to three closely related jobs \cite{jagadeesan2022regret, miller2021outside, maheshwari2022zeroth}, and more recent developments of performative prediction.

\xhdr{Performative Prediction} 
Performative Prediction, first explored in \citet{perdomo2020performative}, has recently received 
many follow-up works, including but not limited to \citet{miller2021outside, izzo2021learn,drusvyatskiy2020stochastic, mendler2020stochastic,brown2020performative, jagadeesan2022regret,dong2021approximate,cutler2021stochastic} and \citet{piliouras2022multi}.
These works mostly focus on the performative stability and the performative optimality, 
including developing an algorithmic procedure that converges 
to performatively stable or optimal points. 
Similar to this line of research \cite{dong2021approximate,jagadeesan2022regret,izzo2021learn,miller2021outside},
our work also focuses on performative optimality.

\xhdr{Zeroth-Order Optimization} 
Our algorithms and techniques are based on the line of work on zeroth-order optimization (also known as bandit optimization) initiated by \citet{flaxman2005online}, which studies how to optimize an unknown convex function $f$ using only function value query access to $f$. 
\citet{agarwal2010optimal} and \citet{shamir2017optimal}  later
extend the technique that allows multiple points query and show that
two points suffice to guarantee that the regret
bounds that closely resemble the regret bounds for the full information case. 
The reparameterization approach proposed in our paper mirrors the intuition behind the algorithms proposed for 
learning from revealed feedback or preferences (see, e.g., \citet{roth2016watch,zadimoghaddam2012efficiently, dong2018strategic}), 
which consider a \emph{Stackelberg game} involving a utility-maximizing learner and strategic agent.
Our work, focusing on performative prediction with an environment response exogenously characterized by a distribution map $\calD(\cdot)$, differs from theirs in problem consideration.
\subsection{Key Notations}

Let $\modelDimen\in \Z_{>0}$ denote the dimension of the model parameter $\theta$, 
and let $\modelDiameter := \sup\{\|\theta-\theta'\|, \forall \theta,\theta' \in \Theta\}$
denote the diameter of the model parameter space $\Theta$. 
The data distribution $\calD(\theta)$ has a parametric 
continuously differentiable density $p(z; \varphi(\theta))$ where $\varphi(\theta)$ denote the distribution parameter for $\calD(\theta)$. We use $\varphi(\cdot)$ to denote the distribution parameter mapping while $\phi$ to denote a given distribution parameter. 
Let $\distributionDimen\in \Z_{>0}$ denote the dimension of the model parameter $\phi$, 
and let $\distributionDiameter := \sup\{\|\phi-\phi'\| ~|, \forall \phi,\phi' \in \Phi\}$
denote the diameter of the model parameter space $\Phi$. When it is clear from the content, we use $\varphi(\theta)$ to represent $\calD(\theta)$ the distribution $\theta$ induces. Let $\vartheta^*(\phi)$ denote the \emph{optimal} model parameter that induces a specific target distribution parameter $\phi$ -- in case of having multiple model parameters that potentially induce the same distribution parameter $\phi$, $\vartheta^*(\phi)$ is the one that achieves the minimum performative prediction risk. 
\vspace{-0.1in}
\subsection{Structure of the Paper} 
The rest of the paper is organized as follows: 
In \Cref{sec:preliminaries}, we introduce the problem formulation and provide a warm-up setting when $\PR(\theta)$ is convex over the model parameter $\theta$.
Using this simple setting, we introduce the zeroth-order optimization technique we use, which will serve as the building block to solve for a more complicated setting (i.e., when $\PR(\theta)$ is 
{\em not} convex over $\theta$). 
We also present a fundamental fact in convex optimization that sublinear regret implies model convergence  (\Cref{proposition:sublinear-regret-implies-convergence}), which unifies the goal of regret minimization and model optimality in our setting. 
In \Cref{subsec: overview}, we provide an overview of our proposed solution.
In \Cref{sec:outer-algorithm}, we describe the outer algorithm, and 
\Cref{sec:learnmodel} describes the inner algorithm called $\learnmodel$, which is used to solve a subroutine problem using black-box oracle. \Cref{sec:putting-things-together} contains the overall regret analysis. Lastly, in \Cref{sec:limitation}, we discuss the limitations and potential future work. 
\textit{All omitted proofs can be found in the Appendix.}
\section{Preliminaries}
\label{sec:preliminaries}
We begin by formally defining our problem. 
\subsection{Problem Formulation}

The objective of performative prediction 
is to minimize the \emph{performative risk} defined as $\PR(\theta):= \E_{z \sim \calD(\theta)} [\ell(z;\theta)]~.$
A model $\theta_{\OPT} \in \Theta$ is said to be \emph{performatively optimal} if $\PR(\theta_{\OPT}) = \min_{\theta \in \Theta} \PR(\theta)$.
To find the performatively optimal model, one usually needs to have the full knowledge of the underlying distribution map $\calD: \Theta \rightarrow \Phi$. 
In this work, we consider a more practical scenario where the distribution map $\calD$ is not known in advance, and to learn the performatively optimal model, the learner has to adaptively deploy models to gradually learn the underlying distribution map. 

Formally, we consider the following repeated interaction between the learner and the environment consisting of decision subjects where we can only query through samples. 
The interaction proceeds for $\totaltime$ steps, 
at each time step $\tau = 1, \ldots, \totaltime$:
(1) the learner deploys a model  $\theta_{\tau}\in\Theta$;
(2) the learner observes $n_\tau$ data samples $\{z_{\tau}^{(i)}\}_{i\in[n_\tau]}$
where each  $z_{\tau}^{(i)}\overset{\text{iid}}{\sim} \calD(\theta_{\tau})$;
(3) the learner incurs empirical loss 
$ \ell(z_{\tau}^{(i)}; \theta_{\tau})$ for each sample.

The goal of the learner is to design an online
model deployment policy $\calA$ such that it minimizes her cumulative empirical risk over all observed data samples:
\begin{align}
    \calR_N(\calA, \PR)
    = \sum_{\tau=1}^{\totaltime}  \sum_{i=1}^{n_{\tau}} \ell(z_{\tau}^{(i)}; \theta_{\tau})- N \cdot \PR(\theta_{\OPT})
\end{align}
where $N:=\sum_{\tau=1}^{\totaltime} n_{\tau}$ denotes the total number of observed data samples throughout the process. The reason we introduce $\totaltime$ instead of $N$ directly is that each step ($\tau$) of our algorithm performs different tasks, where we would impose different requirements of samples to be collected. This shall become clear later when we present our algorithm in the following sections. 

\subsection{Warmup Setting: When $\PR(\theta)$ is Convex in $\theta$}
\label{sec: warmup}
In this section, we analyze a simple scenario when we assume that the performative risk
$\PR(\theta)$ is convex over the model parameter $\theta$. 
The technique we use to solve this simple case will be the building block to solve the later more challenging problem where $\PR(\theta)$ is {\em not} convex over the model parameter $\theta$. 

Recall that when the learner deploys a model $\theta$, 
she observes a set of data samples which are i.i.d drawn from the underlying data distribution $\calD(\theta)$. 
This enables us to compute an unbiased estimate $\PRtilde(\theta)$ 
for the performative risk $\PR(\theta)$
of the deployed model $\theta$:
\begin{align*}
   \PRtilde(\theta) = \frac{1}{n_\tau}\sum_{i=1}^{n_\tau} \ell(z_\tau^{(i)}; \theta), ~~ \text{and}~~ \E[\PRtilde(\theta)] = \PR(\theta), \forall \theta \in \Theta 
\end{align*}
where the expectation is over the randomness of the observed samples.
Since $\PR(\theta)$ is convex over the model parameter $\theta$, 
one can use an off-the-shelf zeroth-order convex optimization technique 
\cite{agarwal2010optimal} to solve this problem and get the following regret guarantee:

\begin{restatable}[]{lem}{convexregretbound}
\label{lemma:regret-bound-for-model-convex-performative-risk}
When $\PR(\theta)$ is convex, $L$-Lipschitz 
w.r.t. the deployed model parameter $\theta$,
there exists an algorithm (\Cref{algorithm:minimize-convex-function})
achieving
$\calR_N(\calA_{\ref{algorithm:minimize-convex-function}}, \PR)
= O(\sqrt{\modelDimen N\log \frac{1}{p}})$
 with probability at least $1-p$, where $N$ is the total number of samples deployed during the process.
\end{restatable}

We defer the proof and the details of \Cref{algorithm:minimize-convex-function} to \Cref{sec:omitted-proofs-for-regret-analysis-algm-1}. In particular, \Cref{algorithm:minimize-convex-function} deploys two models at each time step, in doing so, one can show that the regret bounds closely resemble bounds for the full information case where the learner knows the distribution map $\calD(\cdot)$.
The proof of the above result 
builds on the main result of \citet{agarwal2010optimal}, 
and also incorporates an improved analysis of the gradient estimate due to \citet{shamir2017optimal}.

\subsection{Useful Fact: Sublinear Regret Implies Convergence in Model Optimality}
\label{sec:useful-facts}
A folklore fact in online and zeroth-order optimization is that if a function $f$ is convex and we wish to converge to an approximately optimal point, it suffices to show a query algorithm that achieves $o(n)$ regret after $n$ queries. In particular, we have the following proposition:

\begin{proposition}[Sublinear Regret Implies Convergence]
\label{proposition:sublinear-regret-implies-convergence}
Let $f: X \to \R$ be convex, and let $\calA$ be an algorithm for minimizing $f$ whose regret after $n$ queries is sublinear in $n$, i.e. $\calR_n(\calA,f) = o(n)$. Then we can compute an $\epsilon$-suboptimal point for $f$ in $\calR_n(\calA,f)/\epsilon$ queries of $f$.
\end{proposition}


This proposition establishes a strong link between achieving sublinear regret and the convergence toward an optimal model. It implies that if our proposed algorithm attains a sublinear regret as defined in Equation \ref{eqn:regret-def}, this automatically suggests that we can obtain an almost optimal model, which is exceptionally close to the truly optimal model, denoted as $\theta_{\OPT}$. This closeness is achieved simply by averaging the models $\theta_1, \ldots, \theta_{\totaltime}$ throughout the deployment process. This helps us unify the goal of regret minimization and finding the optimal model.  
\section{Optimizing $\PR$ via Reparameterization: An Overview of Our Solution}
\label{subsec: overview}
When $\PR(\theta)$ is not convex over the model parameter $\theta$,
the zeroth-order convex optimization technique used in \Cref{sec: warmup}
is not directly applicable. 
Instead, we leverage the structure of $\PR(\theta)$
and {\em reparameterize} it as a function of the \emph{induced} data distribution $\calD(\theta)$. 
In particular, we consider the setting where 
the data distribution $\calD(\theta)$ has a parametric 
continuously differentiable density $p(z; \varphi(\theta))$,
and the functional form $p(z; \phi)$ is known to the learner but 
the distribution parameter $\phi$ remains unknown. 
Under mild conditions, we show that 
the performative risk $\PR(\theta)$ can be reformulated as a function of the \emph{induced} distribution distribution parameter $\phi \equiv  \varphi(\theta)$, namely,
\begin{align}
    \PR(\theta) = \PRdagger(\varphi(\theta)) \equiv\PR(\vartheta^*(\phi))~,
\end{align}
and $\PRdagger(\phi)$ is convex over the distribution parameter $\phi$ 
(See more details in \Cref{sec:outer-algorithm}).

Here we provide two real-life settings to justify our model:

\begin{example} (Biased coin flip). Consider the task of predicting the outcome of a biased coin flip similar to \citet{perdomo2020performative}, where the bias of the coin depends on a feature $X$ and the assigned score $f_\theta(X)$. In particular, define $\mathcal{D}(\theta)$ in the following way: $X$ is a 1-dimensional feature supported on $[0, 1]$ and $Y \sim \operatorname{Bernoulli}(\varphi(\theta))$. Assume that the class of predictors consists of linear models of the form $f_\theta(x)= \theta x$ and that the objective is to minimize the squared loss: $\ell(x, y; \theta)=\left(y-f_\theta(x)\right)^2$. When the probability of the coin landing on heads $\varphi(\theta) = \theta^2$, we can verify that $\PR(\theta)$ is convex in $\varphi(\theta)$, not in $\theta$ (by similar argument provided below in \Cref{example:mixture-dominance-too-strong-condition}).
\end{example}

\begin{example}
(Expected revenue of goods). Let $\theta\in \R^d$ denote a vector of prices for various goods the distributor sets. A vector $z$ denotes a customer's demand for each good. The distributor's goal is to maximize the expected revenue $\PR(\theta) = \mathbb{E}_{z\sim \mathcal{D}(\theta)}\left[\theta^{\top} z\right]$. In other words, the loss function is $\ell(z ; \theta)=-\theta^{\top} z$. When $\mathcal{D}(\theta)= N(\varphi(\theta), \Sigma^2)$ with $\varphi(\theta) = \sqrt{\theta}$ and a fixed $\Sigma^2$, we can verify that $\PR(\theta)$ is not convex in $\theta$ but is convex in $\varphi(\theta) = \sqrt{\theta}$.
\end{example}

With this reparameterization, one can operate on the space of distribution parameters
and hopefully apply the zeroth-order convex optimization technique. 
However, one notable challenge is in zeroth-order convex optimization, the learner is usually assumed to have direct query access to the unknown convex function $f$. 
Namely, when querying point $x$, the learner is able to immediately obtain the information about the (noisy) value $f(x)$. 
In our setting, such direct access is, unfortunately, 
not available since the mapping
$\varphi(\cdot)$ is not known to the learner. 
Indeed, the learner can only deploy  a model $\theta$ to observe 
the empirical performative risk $\PRtilde(\theta)$ which is evaluated
over the observed data samples drawn from the induced data distribution $\calD(\theta)$. 
Hence, to evaluate the value $\PRdagger(\phi)$ on a target data distribution with the parameter $\phi$, we use another algorithm called $\learnmodel$ as a subroutine to find a model  $\bar{\theta}$ such that $\varphi(\bar{\theta}) \approx \phi$ (See \Cref{sec:learnmodel}).

\xhdr{Summary of our proposed procedure} Intuitively, the outer loop optimizes the objective function $\PR$ in the distribution parameter space $\phi\in \Phi$ iteratively and tries to find the optimal data parameter $\phi *$, while the inner loop ($\learnmodel$) tries to find a model parameter to induce the particular data parameter that the outer loop is currently iterating on.\footnote{
One may wonder how to find the optimal $\vartheta^*(\phi)$ when there are two model parameters $\theta$ and $\theta^{\prime}$ that realize the same $\phi$ (i.e., $\varphi(\theta)=\varphi\left(\theta^{\prime}\right)$ and $\PR\left(\theta^{\prime}\right) \geq \PR(\theta))$. Recall that the objective function for $\learnmodel$ is to find \emph{any} model $\theta$ that leads to the particular target data parameter $\phi$ such that $\varphi(\theta)=\phi$. It is quite possible that multiple models can induce the same target data parameter; however, since the goal is to find any one of them, having multiple model parameters won't be an issue -- in fact, it can only help speed up the process} 
A graph illustration of our algorithm procedure in given in \Cref{fig:procedure}. 

\begin{figure}[h!]
    \centering
    \includegraphics[width=\linewidth]{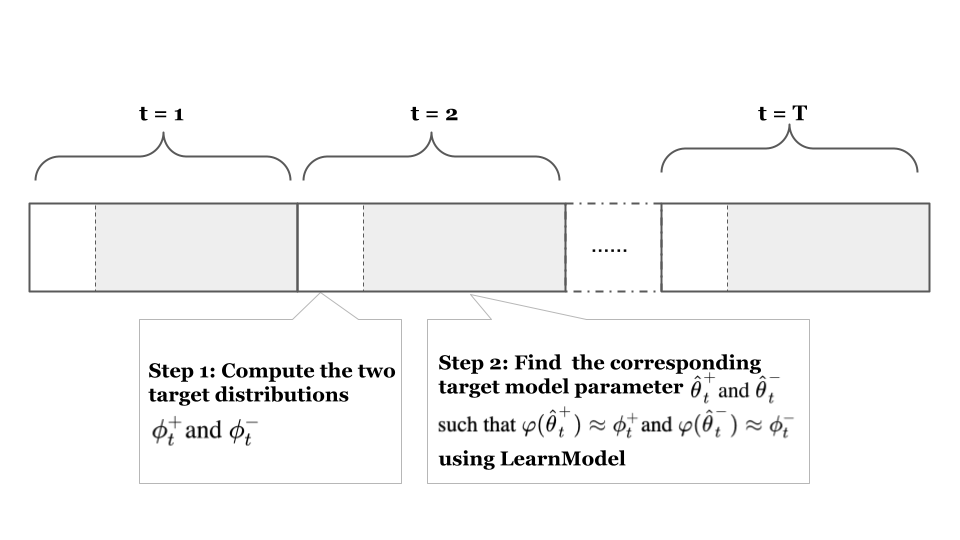}
    \caption{Illustration of our procedure (\Cref{algorithm:minimize-indirectly-convex-function}). Each big block represents one iteration of the outer algorithm, which consists of three sub-steps: Step 1, the learner first computes the two target distribution $\phi^{+}_t$ and $\phi^{-}_t$ (corresponds to the white section), Step 2, the learner uses $\learnmodel$ to learn the corresponding model $\hat{\theta}^{+}_t$ and $\hat{\theta}^{-}_t$ that can best approximately induce $\phi^{+}_t$ and  $\phi^{-}_t$(corresponds to the grey section) correspondingly. Step 3, the learner deploys $\hat{\theta}^{+}_t$ and $\hat{\theta}^{-}_t$ and perform a gradient update and get $\phi_{t+1}$.
    Each deployment of $\learnmodel$ requires a total number of $S$ steps. Thus, the total number of steps involved in the whole procedure is $\totaltime = T\times S$.}
    \label{fig:procedure}
\end{figure}
\vspace{-0.1in}
\subsection{Examples of $\PR$ Being Convex in $\phi$ not in $\theta$}

We first provide three examples in which our condition ($\PR$ loss is convex in the induced distribution parameter $\varphi_\theta:= \phi(\theta)$, but non-convex in the model parameter $\theta$) holds. See more examples and derivation details in \Cref{sec:more-example}.

\begin{example}(Bernoulli distribution)
\label{example:mixture-dominance-too-strong-condition}
Consider the following one-dimension linear model with the squared loss $\ell(\theta;(x,y)) = - (\theta x - y)^2$. 
Assuming a model $\theta\in \Theta = [0,1]$ 
induces a Bernoulli distribution over the labels with the distribution parameter $\varphi(\theta):= \theta^2$, i.e., $y \sim \cc{Bern}(\varphi(\theta))$. Then its $\PR$ loss is convex in its data parameter $\varphi_\theta$ but not convex in its model parameter $\theta$.
\end{example}

Notice that the example provided can be generalized to any distribution map $\phi(\theta)$ that satisfies $\varphi(\theta) = \theta^\alpha$ for any $\alpha> 1$, and any $\ell_\beta$ loss for even $\beta$ value. In addition, Example 1 can also be any generalized monotone polynomial function $\varphi(\theta)$.

\begin{example}[Gaussian distribution]
\label{example:gaussian}
    For a random variable $x$ following a one-dimensional Gaussian distribution with fixed variance, i.e., $\calD(\theta) = N(\varphi(\theta), \sigma^2)$, and let the loss as $\ell(x;\theta) = (\theta x)^2$, then we have $\PR(\theta) = \theta^2(\sigma^2 + \varphi(\theta)^2)$. With $\varphi(\theta) = \sqrt{\theta}$, we verify that $\PR(\theta)$ is not convex in the model parameter $\theta$ but is convex in the distribution parameter $\varphi(\theta) = \sqrt{\theta}$. 
\end{example}

\begin{example}[Uniform distribution]
\label{example:uniform}
  Fix a random variable $x$. Let $y$ follow a uniform distribution with parameter $\varphi(\theta)$, e.g., $y \sim \text{Uniform}[0, \varphi(\theta)]$, and with the loss being $\ell(\theta; x, y) = -(\theta x - y)^2$, we have $\PR(\theta) = \theta^2 x^2 - \theta x \varphi(\theta) + \frac{1}{3} \varphi(\theta)^2$. Setting $\varphi(\theta) = \theta^2$, we verify that it’s convex in the distribution parameter $\varphi(\theta)$ but not in the model parameter $\theta$.  
\end{example}

\begin{remark}
 Earlier work \cite{miller2021outside} posits the ``mixture dominance assumption'', under which the performative prediction risk turns out to be convex in the model parameter $\theta$. In particular, the assumption requires that for any triple $\theta, \theta', \theta_0\in\Theta$, the following condition holds: $\E_{z \sim \calD(\alpha\theta + (1-\alpha)\theta')} [\ell(\theta_0;z)] \le \E_{z \sim \alpha \calD(\theta) + (1-\alpha) \calD(\theta')} [\ell(\theta_0; z)]$.
 The primary distinction between our condition and theirs is that our condition only needs to be valid for each individual data parameter. This is in contrast to \citeauthor{miller2021outside}'s condition, which must be met for any combination of $\theta, \theta'$, and $\theta_0\in \Theta$. We believe our approach has greater versatility and is more likely to be fulfilled in various scenarios \footnote{We can also verify that \Cref{example:mixture-dominance-too-strong-condition} does not satisfy the mixture dominance assumption.}. 
\end{remark}

\section{Our Algorithm and its Performance Guarantee}
In this section, we provide the details of 
our proposed algorithm, and also the associated 
performance analysis.
\subsection{The Outer Algorithm: A Reparameterization Approach}
\label{sec:outer-algorithm}
As we mentioned, in this work, we study the scenario where $\PR(\theta)$ 
is not convex over the model parameter. 
The high-level idea is that we can {\em reparameterize} the performative risk $\PR(\theta)$ as a function $\PRdagger(\phi)$ over the data distribution parameter $\phi$.
In particular, we first reformulate the learner's loss function so that it can be expressed as a function \emph{only} in the induced data distribution. For each data distribution
$\phi \in \Phi$, assume the set of learner's actions (deployed model parameters) that induce $\phi$ is $ \Theta^*(\phi) = \{\theta \in \Theta | \varphi(\theta) = \phi  \}$
Among all of the learner's actions that induce $\phi$, the optimal one that achieves the minimal $\PR$ loss across the whole population is:
\vspace{-0.1in}
\begin{align*}
    \vartheta^*(\phi) = \argmin_{\theta\in \Theta^*(\phi )} \PR(\theta )
\end{align*}
where ties are broken arbitrarily. 
Now we can rewrite 
learner's objective function as a function
of $\phi$
\begin{align}
\label{eq:reformulate-objective-function}
    \PRdagger(\phi) = \PR(\vartheta^*(\phi ))
\end{align}

To make the problem tractable, we consider following generic
class of $\PRdagger(\cdot)$ that is convex and Lipchitz continuous. 
\begin{assumption} 
\label{ass:PRdagger-convex-lipschitz}
$\PRdagger(\phi)$ is convex and $L^\dagger$-Lipschitz over
the data distribution parameter $\phi\in \Phi$. 
\end{assumption}


With reparameterizing $\PR(\theta)$ as a function $\PRdagger(\phi)$ 
over the induced data distribution parameter $\phi$, 
we now wish to minimize a bounded, $L^\dagger$-Lipschitz function
$\PRdagger(\cdot): \Phi \to \R $, where $\Phi \subset \R^{\distributionDimen}$ has bounded diameter $\distributionDiameter$,
by operating on the distribution parameter space $\Phi$.

Instead of having immediate query access in 
zeroth-order convex optimization algorithm, 
in our setting, we cannot directly evaluate the (noisy) value 
$\PRdagger(\phi)$ for a particular data distribution parameter, but may query the following oracles:
\squishlist
    \item A noisy \emph{function oracle} $\estimatePR$, which takes  $\theta \in \Theta$ as input and returns an unbiased estimate $\PRtilde$  such that $\E[\PRtilde(\theta)] = \PR(\theta)$. This noisy oracle can be implemented by simply making a prediction and observing the loss 
    as defined in \Cref{sec: warmup} \footnote{The sample required for each round of estimate EstimatePR is $\mathcal{O}(1)$; this is because EstimatePR itself is an unbiased estimator, so even with one sample, in expectation, the estimation will be unbiased.}.
    \item A noisy \emph{reparameterization oracle} $\learnmodel(\phi, \epsilon_\LM, p_\LM)$, which takes $\phi \in \Phi$, $\epsilon_\LM, p_\LM > 0$ as input and returns $\theta \in \Theta$ such that $\Pr(\|\varphi(\theta) - \phi\| \geq \epsilon_\LM) \leq p_\LM$. We will specify $\learnmodel$ in \Cref{sec:learnmodel}.
\squishend

\Cref{algorithm:minimize-indirectly-convex-function} achieves this task. Specifically, it returns both $\bar{\theta} \in \Theta$ and $\bar{\phi} \in \Phi$ such that with probability at least $1-p$, $|\PR(\bar{\theta}) - \PR(\theta_\OPT)| \leq \epsilon$ and $|\PRdagger(\bar{\phi}) - \PR^\dagger(\varphi(\theta_\OPT))| \leq \epsilon$.

\begin{algorithm}[h!]
\small
\begin{algorithmic}
\caption{Bandit algorithm for minimizing an indirectly convex function with noisy oracles}
\label{algorithm:minimize-indirectly-convex-function}
\Function{$\estimatePR$}{$\theta$} \Comment{\textcolor{blue}{Unbiased estimate of $\PR(\theta)$}}
        \State Deploy $\theta$, observe sample $z \sim \calD(\theta)$
        \State \Return $\ell(z;\theta)$
\EndFunction

\Function{$\minimizePR$}{$\learnmodel: \Phi \to \Theta$; $\epsilon, p, \epsilon_\LM, p_\LM > 0$}
        \State $T \gets \frac{\distributionDimen}{(\epsilon - \sqrt{\epsilon_\LM \distributionDimen})^2}$, $\delta \gets \sqrt{\epsilon_\LM \distributionDimen}$, $\eta \gets 1/\sqrt{\distributionDimen T}$
        \State $y_1 \gets \mathbf{0}$
        \For{$t \gets 1,\ldots,T$}
            \State $u_t \gets$ sample from $\mathrm{Unif}(\S)$
            \State $\phi_t^+ \gets \phi_t + \delta u_t$, $\phi_t^- \gets \phi_t - \delta u_t$
            \State $\hat{\theta}_t^+ \gets \learnmodel(\phi_t^+, \epsilon_\LM, p_\LM)$
            \State $\hat{\theta}_t^- \gets \learnmodel(\phi_t^-, \epsilon_\LM, p_\LM)$
                \Comment{\textcolor{blue}{$\hat{\theta}_t^+$ such that $\PR(\hat{\theta}_t^+) \approx \PRdagger(\phi_t^+)$, similarly, $\hat{\theta}_t^-$ such that $\PR(\hat{\theta}_t^-) \approx \PRdagger(\phi_t^-)$}}
            \State $\PRtilde(\hat{\theta}_t^+) \gets \estimatePR(\hat{\theta}_t^+)$
            \State    $\PRtilde(\hat{\theta}_t^-) \gets \estimatePR(\hat{\theta}_t^-)$
                \Comment{\textcolor{blue}{Approximations of $\PR(\hat{\theta}_t^+)$, $\PR(\hat{\theta}_t^-)$}}
            \State $\tilde{g}_t \gets \frac{\distributionDimen}{2\delta} \left(
                    \PRtilde(\hat{\theta}_t^+) - \PRtilde(\hat{\theta}_t^-)
                \right) \cdot u_t$
                \Comment{\textcolor{blue}{Approximation of $\nabla_\phi \PRdagger(\phi_t)$}}
            \State $\phi_{t+1} \gets \Pi_{(1-\delta)\Phi}(\phi_t - \eta \tilde{g}_t)$
                \Comment{\textcolor{blue}{Take gradient step and project}}
        \EndFor
        \State $\bar{\phi} \gets \frac{1}{T} \sum_{t=1}^T \phi_t$
        \State $\bar{\theta} \gets \learnmodel(\bar{\phi},\epsilon_\LM,p_\LM)$
        \State \Return $\bar{\theta}$, $\bar{\phi}$
    \EndFunction
\end{algorithmic}
\end{algorithm}

For analysis purpose, we also define regret in $T$, the total number of steps $\minimizePR$ has to go through in order to get an $\epsilon$-suboptimal model parameter w.r.t the $\PR$ objective function:
\begin{align*}
    &\calR_T(\minimizePR, \PR)\\
    =& \sum_{t=1}^T \left[
            \estimatePR(\hat{\theta}_t^+)
            + \estimatePR(\hat{\theta}_t^-)
            - 2 \PR(\theta_\OPT)
        \right]
\end{align*}

We demonstrate the following regret bound for this algorithm:



\begin{restatable}[High-probability regret bound for \Cref{algorithm:minimize-indirectly-convex-function} in $T$]{thm}{outeralgorithmregret}
\label{theorem:regret-bound-for-indirectly-convex-functions}
When \Cref{algorithm:minimize-indirectly-convex-function} is called with arguments $\epsilon_\LM$ and $p_\LM$, we have for every $p > 0$ that
{\small
\begin{align*}
 \calR_T(\minimizePR,\PR) 
    =  ~ O\left(
        \sqrt{\distributionDimen T}
        + \sqrt{\epsilon_\LM \distributionDimen} \cdot T
        + \sqrt{T \log\frac{1}{p}}
    \right)
\end{align*}
}
with probability at least $1 - p - 2Tp_\LM$.
\end{restatable}

The above \Cref{theorem:regret-bound-for-indirectly-convex-functions} requires that the output of 
$\learnmodel$ is $\epsilon_\LM$-close 
to the target distribution parameter $\phi$
with probability at least $1 - p_\LM$. 
Later in \Cref{sec:learnmodel}, 
we show how we achieve this by developing an 
zeroth-order convex optimization algorithm with the 
objective of minimizing the $\KL$ divergence of two distributions.

\subsection{Inner Algorithm: Inducing a Target Distribution Using $\learnmodel$}
\label{sec:learnmodel}
In this section, we show how to solve the sub-problem $\learnmodel$ mentioned in \Cref{algorithm:minimize-indirectly-convex-function}: given a target distribution with the parameter $\phi\in \Phi$, find a model $\theta\in \Theta$ whose corresponding distribution parameter $\varphi(\theta)$ is close to $\phi$.

\xhdr{Objective function for $\learnmodel$} 
To this end, we consider minimizing 
the $\KL$ divergence between $\phi$ and $\varphi(\theta)$:
\footnote{For notation simplicity, here, we use $\KL(\phi_1|| \phi_2)$ 
to represent $\KL(\calD_1|| \calD_2)$ where the data distribution $\calD_1$
and $\calD_2$ has the parameter $\phi_1$ and $\phi_2$, respectively.}
\begin{align}
\label{eqn:kl-divergence}
    \KL(\phi|| \varphi(\theta)) :=  \int_z p(z; \phi) \log \frac{p(z; \phi)}{p(z;\varphi(\theta))} dz
\end{align}
where $p(z;\phi)$ denotes the pdf for the target distribution $\phi$, and $p(z;\varphi(\theta))$ denotes the pdf for the distribution induced by deploying $\theta$.

In general, ${\KL}(\phi|| \varphi(\theta))$ measures how much a distribution with the parameter $\varphi(\theta)$ is away from the  target distribution with the parameter $\phi$: if the two distributions $\phi_1,\phi_2 \in \Phi$ satisfy $\phi_1 = \phi_2$, 
then ${\KL} (\phi_1||\phi_2) = 0$, otherwise ${\KL} (\phi_1||\phi_2) >0$. 
Intuitively, the lower the value ${\KL} (\phi_1||\phi_2)$ is, the better we have matched the target distribution with our approximate 
distribution induced by the chosen model.
However, $\KL(\phi|| \cdot )$ is generally not convex nor Lipschitz.
Hence, to make the problem tractable, we will make several assumptions. 
We view these assumptions as comparatively mild, and provide 
examples shortly after stating the assumptions we need. 

\begin{assumption}
\label{ass:kl-convex-lipschitz}
The function $\KL(\phi|| \varphi(\cdot))$, the data distribution $\calD(\theta)$,
and its parameter mapping $\distParaMap(\cdot)$ satisfies the following properties.
\begin{enumerate}
    \item[2a.]  $\KL(\phi|| \varphi(\cdot))$ is convex in the model parameter $\theta\in \Theta$;
    \item[2b.]  The data distribution $\calD(\theta)$ with the parameter $\distParaMap(\theta)$ is $(\ell_2, K)$-Lipschitz continuous in the model parameter $\theta \in \Theta$ with constant $K(z), \forall z\in Z$
    \footnote{A distribution $\calD(\theta)$ with the density function $p(\cdot | \varphi(\theta))$ parameterized by $\theta \in \Theta$ is called $(\ell_2, K)$-Lipschitz continuous \cite{Honorio2012lipschitz} if for all $z$ in the sample space, the log-likelihood $f(\theta) = \log p(z|\varphi(\theta))$ is Lipschitz continuous with respect to the $\ell_2$ norm of $\theta$ with constant $K(z)$. 
    };
    \item[2c.]   
    \label{ass:phi-difference-bound-by-dtv}
    Let $\calD_1, \calD_2$ be two data distributions 
    with the parameter $\phi_1, \phi_2 \in \Phi$, 
    and $d_\TV (\calD_1, \calD_2)$ be the total variation distance.
    Then $\|\phi_1 - \phi_2\|  \leq L_\TV\cdot d_\TV (\calD_1, \calD_2) $
    for some constant $L_\TV>0$.
\end{enumerate}
\end{assumption}



Here, we provide examples to demonstrate that the above assumptions
are comparatively mild.
The following is an example showing the convexity of 
$\KL(\phi|| \varphi(\cdot))$.


\begin{restatable}[]{exam}{klconvexexample}
\label{example:convex-kl}
Consider the density function $p(z; \varphi(\theta))$ 
of the data distribution $\calD(\theta)$ 
satisfying $p(z; \varphi(\theta)) = \mathrm{Unif}(\exp(c\varphi(\theta)))$ 
for some constant $c > 0$
and for any convex function $\varphi(\theta)$, 
then $\KL(\phi|| \varphi(\cdot))$ is convex over $\theta$.
\end{restatable}

In the above Assumption 2b, we assume a family of distribution called the $(\ell_2, K)$-Lipschitz continuous. 
This Lipschitz continuity over the 
parametrization of probability distributions allows us to have 
the following Lipschitz condition of the function $\KL(\phi||\varphi(\cdot))$
over the model parameter $\theta$: 


\begin{restatable}[Lipschitzness of $\KL(\phi||\varphi(\theta))$ in $\theta$]{lem}{kllipschitzcondition}
\label{lemma:lip-phi-in-KL}
Given two $\left(\ell_{2}, K\right)$-Lipschitz continuous distributions $\calD_{1}=p\left(\cdot \mid \varphi(\theta_{1})\right)$ and $\calD_{2}=p\left(\cdot \mid \varphi(\theta_{2})\right)$, and a target distribution parameter $\phi\in \Phi$, we have
$\left|\KL \left(\phi||\varphi(\theta_{1})\right) - \KL\left(\phi|| \varphi(\theta_{2})\right) \right|\leq {L_\KL}\left\|\theta_{1}-{\theta}_{2}\right\|$
with a constant $L_\KL>0$.
\end{restatable}

The above Assumption 2c is about the continuity on the distribution parameter $\phi\in \Phi$. 
Intuitively, this assumption ensures that if the parameters of two distribution  are close, then their total variation distance is close as well.
With this assumption, we can show that the distance between two 
distribution parameters $\|\phi_1 - \phi_2\|$ can be bounded by 
the KL divergence between the corresponding data distributions.  

\begin{restatable}[]{lem}{phiboundbykl}
\label{lemma:phi-bound-by-KL}
    With Assumption 2c, we have $\|\phi_1 - \phi_2\|\leq L_\phi \sqrt{\KL(\phi_1||\phi_2)}$ for some constant $L_\phi>0$.
\end{restatable}

Intuitively, the above result ensures that given 
a target distribution parameter $\phi$, 
as long as a model $\theta$ whose corresponding data distribution
is close (i.e., $\KL(\phi || \varphi(\theta))$ is small) to the distribution with the parameter $\phi$, then 
$\varphi(\theta)$ is close to $\phi$.
We will use \Cref{lemma:phi-bound-by-KL} in the proof of our main theorem in \Cref{sec:putting-things-together}.


\xhdr{Algorithm for $\learnmodel$}
When $\KL(\phi|| \varphi(\cdot))$ is convex and Lipschitz 
over the model $\theta$, its minimizer can be computed using algorithms similar to \Cref{algorithm:minimize-indirectly-convex-function}. 
In our problem, given a target data distribution with the 
parameter $\phi$, we can use the observed data samples to approximately compute the $\KL(\phi|| \varphi(\theta))$ when deploying a model $\theta$. 
Indeed, we assume an existence of an oracle $\estimatekl(\phi, (z_t^{(i)})_{i\in[n_t]})$
which takes the observed samples $(z_t^{(i)})_{i\in[n_t]}$ realized 
from the induced data distribution $\calD(\theta)$
and the target data distribution parameter $\phi$ as input
to approximate the value $\KL(\phi|| \varphi(\theta))$.
We remark that such oracle has been widely used in the literature on KL divergence estimation \cite{rubenstein2019practical}. 

\begin{definition}[Oracle $\estimatekl$]
\label{ass:kl-oracle}
There exists an oracle $\estimatekl$ that given any target parameter $\phi\in \Phi$, error tolerance $\epsilon_\KL > 0$ and error probability $p_\KL > 0$, and $N_\KL(\epsilon_\KL, p_\KL)$ samples $z_1,\ldots,z_{N_\KL(\epsilon_\KL, p_\KL)}$ from a distribution with parameter $\phi'$, returns an estimated $\KL$ divergence $\widetilde{\KL}(\phi||\phi')$ satisfying $\big\|\widetilde{\KL}(\phi||\phi') - \KL(\phi||\phi')\big\|\leq \epsilon_\KL$ with probability at least $1 - p_\KL$.
\end{definition}

With the oracle $\estimatekl$ to approximately compute
the KL divergence, we are now ready to present our inner algorithm  $\learnmodel$ (see \Cref{algorithm:learnmodel}).

\begin{algorithm}[h!]
\small
\begin{algorithmic}
\caption{Learn a model that approximately induces a given distribution parameter $\phi$}
\label{algorithm:learnmodel}
   \Function{$\learnmodel$}{$\phi\in \Phi$; $\epsilon_\LM, p_\LM>0$, $\epsilon_\KL, p_\KL>0$}
        \State $S \gets \frac{\modelDimen}{(\epsilon_\LM - \sqrt{\epsilon_\KL \modelDimen})^2}$, $\delta_\LM \gets \sqrt{\epsilon_\KL \modelDimen}$
        \State $\eta_\LM \gets \frac{1}{\sqrt{\modelDimen S}}$, $N_\KL \gets N_\KL(\epsilon_\KL, p_\KL)$
        \State $\theta_1 \gets \mathbf{0}$
        \For{$s \gets 1,\ldots,S$}
            \State $u_s \gets$ sample from $\mathrm{Unif}(\S^{\modelDimen})$
            \State $\theta_s^+ \gets \theta_s + \delta_\LM u_s$,
                $\theta_s^- \gets \theta_s - \delta_\LM u_s$
            \State $z^+_{s,1:N} \sim \varphi(\theta_s^+)$,
                $z^-_{s,1:N} \sim \varphi(\theta_s^-)$
                \Comment{\textcolor{blue}{Deploy $\theta_s^+$, $\theta_s^-$; observe $N_\KL$ samples}}
            \State $\widetilde{\KL}\left(\phi||\varphi(\theta_s^+)\right) \gets \estimatekl(\phi, z^+_{s,1:N},\epsilon_\KL, p_\KL)$
            \State $\widetilde{\KL}\left(\phi||\varphi(\theta_s^-)\right) \gets \estimatekl(\phi, z^-_{s,1:N},\epsilon_\KL, p_\KL)$
                \Comment{\textcolor{blue}{Approximations of $\KL$}}
            \State $\tilde{g}_s \gets \frac{\modelDimen}{2\delta_\LM} \left( \widetilde{\KL}(\phi||\varphi({\theta}_s^+)) - \widetilde{\KL}(\phi||\varphi(\theta_s^-)
                \right) \cdot u_s$
                \Comment{\textcolor{blue}{Approximation of $\nabla_\theta \KL(\phi||\varphi(\theta_s))$}}
            \State $\theta_{s+1} \gets \Pi_{(1-\delta_\LM)\Theta}(\theta_s - \eta_\LM \tilde{g}_s)$
                \Comment{\textcolor{blue}{Take gradient step and project}}
        \EndFor
        \State $\bar{\theta} \gets \frac{1}{S} \sum_{s=1}^S \theta_s$
        \State \Return $\bar{\theta}$
    \EndFunction

\end{algorithmic}
\end{algorithm}

Similar to before, for analysis purpose, we also define regret of $\learnmodel$ in $S$, the total number of rounds $\learnmodel$ has to go through in order to output a $\epsilon_\LM$-suboptimal model parameter w.r.t the $\KL$ objective function: 
{
\begin{align*}
   & \calR_S(\learnmodel, \KL) \\
    = ~ & \sum_{s=1}^S 
    \left[
        \widetilde{\KL}(\phi||\varphi(\theta_{s}^+)) + \widetilde{\KL}(\phi||\varphi(\theta_{s}^-))
        - 2\KL(\phi||\vartheta^*(\phi))
    \right]
\end{align*}
}
where $\vartheta^*(\phi)$ is the model that can induce the target distribution $\phi$. 
Using the similar arguments in \Cref{theorem:regret-bound-for-indirectly-convex-functions}, we first show the following regret guarantee for $\learnmodel$:


\begin{restatable}[High-probability regret bound for \Cref{algorithm:learnmodel} with $S$ rounds]{thm}{learnmodelregret}
\label{theorem:regret-bound-for-learnmodel}
When $\learnmodel$ is run for $S$ steps and invokes $\estimatekl$ with arguments $\epsilon_\KL > 0$ and $p_\KL > 0$, we have 
$\forall p > 0$
{\small
\begin{align*}
     \calR_S(\learnmodel,\KL)
    = ~   O\left(
        \sqrt{d_\Phi S}
        + \sqrt{\epsilon_\KL d_\Phi} \cdot S
        + \sqrt{S \log\frac{1}{p}}
    \right)
\end{align*}
}
with probability at least $1 - p - 2S p_\KL >0$.
\end{restatable}
\Cref{theorem:regret-bound-for-learnmodel} characterizes the regret as a function of the total number of deployments of the procedure in $\learnmodel$. Together with regret characterization of the outer algorithm in \Cref{theorem:regret-bound-for-indirectly-convex-functions}, 
we can get the final regret bound.

\vspace{-0.1in}
\subsection{Putting All Pieces Together}
\label{sec:putting-things-together}

As shown in the previous section, both the outer algorithm ($\minimizePR$ -- in \Cref{sec:outer-algorithm}) and inner algorithm ($\learnmodel$ -- in \Cref{sec:learnmodel}) achieve a sublinear regret w.r.t the total number of steps ($T$ and $S$) when outputting an $\epsilon$-optimal solutions. 
In this section, we combine the results in \Cref{sec:outer-algorithm} and \Cref{sec:learnmodel} 
to conclude the analysis for $\minimizePR$ (\Cref{algorithm:minimize-indirectly-convex-function}) for convex $\PRdagger(\phi)$. 
The main result of this section is summarized as follows:

\begin{restatable}[Regret of $\minimizePR$ in $N$]{thm}{totalregret}
\label{theorem:total-regret}
Under Assumption \ref{ass:kl-convex-lipschitz}, and given
access an oracle $\estimatekl$, there exists a choice of $\epsilon_\KL, p_\KL > 0$ in \Cref{algorithm:learnmodel} such that for every $p > 0$,
{\small
\begin{align*}
    &\calR_N(\minimizePR, \PR)\\
    = ~ &\widetilde{O}\left(
            (d_\Theta + d_\Phi)
            N_\KL(\epsilon_\KL,p_\KL)^{1/6}
            N^{5/6}
            \sqrt{\log\frac{1}{p}}
        \right)
\end{align*}
}
with probability at least $1-p$.
\end{restatable}

\begin{proof}[Proof Sketch of \Cref{theorem:total-regret}]
Let $T$ be the number of steps executed by the outer algorithm $\minimizePR$, and $S$ the number of steps in $\learnmodel$. Let $N_\KL(\epsilon_\KL, p_\KL)$ (or $N_\KL$ for short) denote the number of samples used by $\estimatekl$. Since $\minimizePR$ calls $\estimatePR$ and $\learnmodel$ $2T$ times, and $\learnmodel$ calls $\estimatekl$ $2S$ times, the overall number of samples involved in the whole process is $N = 2(2N_\KL S + 1)T$. Following the regret definition, we can break down the regret into the regret from calling $\estimatePR$ in the outer algorithm and the regret from calling $\estimatekl$ in $\learnmodel$. Using the fact that $\PRdagger$ is Lipschitz in the distribution parameter $\phi$ and the distance between any two 
distribution parameters can be bounded by the KL divergence between the corresponding
data distributions (\Cref{lemma:phi-bound-by-KL}), we show that the total regret in $N$ can be expressed as:
{
\begin{align*}
  &\calR_N(\minimizePR, \PR)\\
 =&  O\left(
    \sqrt{N} + N_\KL T \cdot \sqrt{S \cdot \calR_S(\learnmodel, \KL)}\right.\\
    & \quad \quad \left. + (N_\KL S + 1) \cdot \calR_T(\minimizePR, \PR)\right)    
\end{align*}
}
where $\calR_T(\minimizePR, \PR)$ and $\calR_S(\learnmodel, \KL)$ are obtained from \Cref{theorem:regret-bound-for-indirectly-convex-functions} and \Cref{theorem:regret-bound-for-learnmodel} as functions of $\epsilon_\LM, \epsilon_\KL, S, T$ and $\modelDiameter$ and $\distributionDiameter$. Then by balancing the terms and setting $\epsilon_\LM$ and $\epsilon_\KL$ according to the convergence analysis for both $\minimizePR$ and $\learnmodel$ (\Cref{claim:convergence-of-minimizePR} and \Cref{claim:convergence-of-learnmodel-KL}), we can get an express of the total regret. 
\end{proof}
\vspace{-0.1in}
\Cref{theorem:total-regret} show that our procedure is sublinear in the $N$, the total number of samples we deploy during the process. Notice that this also implies that our method is sublinear w.r.t. the total number of deployments $\totaltime = S \times T$. To see this, recall that the total number of samples $N$ required throughout the process is $N =2(2N_{{kl}} \times S + 1) T$; if we measure the regret w.r.t. the total number of deployments, $N_{{kl}}$ will be constant, and thus the regret will still be sublinear w.r.t. $\totaltime$. In addition, recall from \Cref{proposition:sublinear-regret-implies-convergence}, the lowest regret optimal classifier implies the optimal classifier up to some addictive error, which means that by having a sublinear regret in $N$ (and $\totaltime$), we also get a model that is arbitrarily close to the performative optimal model.  

\paragraph{Empirical Evaluation Using Toy Example}
We provide empirical results using a toy example to demonstrate the efficiency of our method. In particular, we compare our proposed method (which minimizes $\text{PR}$ as a function of the distribution parameter $\phi$ after reparametrization) with the baseline method (which directly minimizes $\PR$ as a model parameter $\theta$). We observe that under different settings, both methods converge. However, our proposed method (shown in orange) is more efficient: it demonstrates a much faster convergence rate on average over multiple runs, indicating that our reparametrization method is effective when dealing with distributions that are non-convex in $\theta$ but convex in $\phi$ (as per Assumption 1). The plot can be found in \cref{sec:empirical}. The details for reproducing our experimental results can be found at  {\small \texttt{\url{https://github.com/UCSC-REAL/PP-bandit-feedback}}}.

\section{Practical Consideration and Future Works}
\label{sec:limitation}
Since our method uses a double-loop zeroth order optimization method, the convergence is likely to be slow in practice. Thus, it requires extra consideration regarding the convergence rate, and efficiency can vary based on particular settings. In particular, the success of our method depends on the fast deployment of the frequently updated models. 
One potential way to speed up the deployment process may be performing \emph{parallelization}, which accelerates the optimization process and reduces the time required to find a suitable solution (see, e.g., \cite{liu2020primer} for a detailed reference). In our algorithm, parallelization can be applied to several places, e.g., the two-point estimations can be computed in parallel and potentially speed up the process.

\section*{Acknowledgements} Y. Chen and Y. Liu are partially supported by the National Science Foundation (NSF) under grants IIS-2143895 and IIS-2040800.

\section*{Impact Statement}
Since the process of finding the optimal performative model involves deploying sub-optimal models on human agents, sample efficiency is important. Thus sample complexity needs to be taken into consideration when choosing the particular KL divergence oracle used in $\learnmodel$. Additionally, the concept of performativity highlights a significant broader impact of our work: the importance of recognizing scenarios where predictions can modify the very environment they're meant to predict. Take, for example, an online advertising platform utilizing machine learning models to tailor ads for users. These models, by analyzing user behaviors and traits to serve personalized ads, might inadvertently influence both user and advertiser actions, thereby establishing a feedback loop. This dynamic underscores the need for a mindful approach in deploying predictive models, especially in settings sensitive to the outcomes of such predictions.

\newpage


\bibliography{reference}
\bibliographystyle{icml2024}

\newpage
\appendix
\onecolumn
\section{Appendix Arrangement}
We arrange the appendix as follows:
\squishlist
    \item \Cref{sec:omitted-proofs-for-regret-analysis-algm-1} provides omitted algorithm and proofs for \Cref{sec:preliminaries}.
    \item \Cref{sec:more-example} provides more examples that satisfies $\PR$ being convex in $\phi$ not in $\theta$. 
    \item \Cref{proof:outer-algorithm} provides
    omitted example and proofs for \Cref{sec:outer-algorithm}.
    \item \Cref{sec: omitted-proof-learnmodel} provides omitted proofs for \Cref{sec:learnmodel}.
    \item \Cref{sec:proof-put-things-together} provides omitted proof for \Cref{sec:putting-things-together}.
    \item \Cref{sec:additional-related-work} provides an additional literature review on performative prediction and related literature. 
    \item \Cref{sec:empirical} provides empirical verification using toy example generated by \Cref{example:mixture-dominance-too-strong-condition} to demonstrate the efficiency of our proposed method.
\squishend


\section{Omitted Algorithm and Proof for \Cref{sec:preliminaries}}
\label{sec:omitted-proofs-for-regret-analysis-algm-1}

\subsection{Omitted Proof for \Cref{proposition:sublinear-regret-implies-convergence}}

\begin{proof}
Let $x_1,\ldots,x_n$ be the first $n$ points queried by $\calA$. By the convexity of $f$, the average of these points $\bar{x} = \frac{1}{n} \sum_i x_i$ satisfies
\begin{align*}
    f(\bar{x}) - f(x^*)
    \leq \frac{1}{n} \sum_{i=1}^n \left[
            f(x_i) - f(x^*)
        \right]
    = \frac{\calR_n(\calA,f)}{n}
\end{align*}
Thus if $\calR_n(\calA,f) = o(n)$, then after $n = \calR_n(\calA,f)/\epsilon$ queries, $\bar{x}$ satisfies $f(\bar{x}) - f(x^*) \leq \epsilon$ as required.
\end{proof}

\subsection{Omitted algorithm and proof for \Cref{algorithm:minimize-convex-function} }

\Cref{algorithm:minimize-convex-function} 
is a straightforward generalization of the algorithm introduced by \cite{agarwal2010optimal}, 
while we generalize their setting where the function 
can be evaluated exactly
to the setting where noisy evaluation is allowed.

\begin{algorithm}[H]
\begin{algorithmic}
\caption{Bandit algorithm for minimizing convex and lipschitz $\PR(\theta)$ }
\label{algorithm:minimize-convex-function}
    \Function{$\estimatePR$}{$\theta$} \Comment{Unbiased estimate of $\PR(\theta)$}
        \State Deploy $\theta$, observe sample $z \sim \calD(\theta)$
        \State \Return $\ell(z;\theta)$
    \EndFunction

    \Function{MinimizePR}{$T$}
        \State $\delta \gets \sqrt{d_\theta/T}$
        \State $\eta \gets 1/\sqrt{\modelDimen T}$
        \State $\theta_1 \gets \mathbf{0}$
        \For{$t \gets 1,\ldots,T$}
            \State $u_t \gets$ sample from $\mathrm{Unif}(\S^{\modelDimen})$
            \State $\theta_t^+ \gets \theta_t + \delta u_t$, $\theta_t^- \gets \theta_t - \delta u_t$
            \State $\PRtilde(\theta_t^+) \gets \estimatePR(\theta_t^+)$
                \Comment{Approximations of $\PR(\theta_t^+)$, $\PR(\theta_t^-)$}
            \State $\PRtilde(\theta_t^-) \gets \estimatePR(\theta_t^-)$
            \State $g_t \gets \frac{\modelDimen}{2\delta} \left(\PRtilde(\theta_t^+) - \PRtilde(\theta_t^-)\right) \cdot u_t$
                \Comment{Approximation of $\nabla_\theta \widehat{\PR}(\theta_t)$}
            \State $\theta_{t+1} \gets \Pi_{(1-\delta)\Theta}(\theta_t - \eta g_t)$
                \Comment{Take gradient step and project}
        \EndFor
        \State \Return $\frac{1}{T} \sum_{t=1}^T \theta_t$
    \EndFunction
\end{algorithmic}
\end{algorithm}

To prove \Cref{lemma:regret-bound-for-model-convex-performative-risk}, we first provide a series of lemmas and claims that will be useful later.

\begin{claim}[Regret from estimating $\PR$]
\label{claim:regret-from-estimating-f}
For any $p>0$, with probability at least $1-p$,
\begin{align*}
    \sum_{t=1}^T \left[
            \widetilde{\PR}(\theta_t^+) - f(\theta_t^+)
        \right]
    \leq F\sqrt{T \log\frac{1}{p}}
    \quad \mathrm{and} \quad
    \sum_{t=1}^T \left[
            \widetilde{\PR}(\theta_t^-) - f(\theta_t^-)
        \right]
    \leq F\sqrt{T \log\frac{1}{p}}
\end{align*}
\end{claim}
\begin{proof}
The claim follows from Hoeffding's inequality, since $\estimatePR$ is unbiased and bounded by $[0,F]$.
\end{proof}

\begin{claim}[Regret from smoothing over the sphere or ball]
\label{claim:regret-from-smoothing-over-sphere-or-ball}
For any $\theta \in \Theta$, $u \in \S$, and $\delta > 0$, all of the following are at most $\delta L$:
\begin{align*}
    |\PR(\theta + \delta u) - \PR(\theta)|,
    \quad
    |\PR(\theta - \delta u) - \PR(\theta)|,
    \quad\\
    \left|
            \frac{1}{2}[
                \PR(\theta + \delta u)
                + \PR(\theta - \delta u)
            ]
            - \PR(\theta)
        \right|,
    \quad \text{and} \quad
    |\widehat{\PR}(\theta) - \PR(\theta)|.
\end{align*}
\end{claim}
\begin{proof}[Proof sketch]
Lipschitzness of $\PR$.
\end{proof}

\begin{claim}[Deviation of smoothed function]
\label{claim:deviation-of-smoothed-function}
For any $p > 0$, with probability at least $1-p$,
\begin{align*}
    \sum_{t=1}^T \widehat{\PR}(\theta_t)
        - \E_T \left[
            \sum_{t=1}^T \widehat{\PR}(\theta_t)
        \right]
    \leq F \sqrt{T \log\frac{1}{p}}
\end{align*}
\end{claim}
\begin{proof}[Proof sketch]
The left-hand side is the sum of a martingale difference sequence. The Azuma-Hoeffding inequality yields the result.
\end{proof}

\begin{claim}[Gradient estimate is unbiased and bounded]
\label{claim:gradient-estimate-is-unbiased-and-bounded}
There exists a constant $c>0$ such that for all $t \in [T]$, $\E_t[g_t] = \nabla \widehat{\PR}(\theta_t)$ and $\|g_t\|_2^2 \leq cd_\theta L^2$.
\end{claim}
\begin{proof}
Proved in \citeauthor{shamir2017optimal} (see Lemma 10, noting that the $\ell_2$ norm is its own dual).
\end{proof}


\begin{restatable}[Expected suboptimality under smoothing when $\PR$ is convex]{lem}{}
\label{lemma:expected-suboptimality-under-smoothing}
Let $\theta \in \Theta$, and let $\theta_1,\ldots,\theta_t \in \Theta$ be a sequence of iterates given by the update rule $\theta_{t+1} = \Pi_{(1-\delta)\theta} (\theta_t - \eta g_t) - \theta$ for some sequence of gradient estimates $g_t \in \R^{d_\Theta}$. Then
\begin{align*}
    \E_T \left[\sum_{t=1}^T \widehat{\PR}(\theta_t)\right]
        - \sum_{t=1}^T \widehat{\PR}(\theta)
    \leq \frac{D_\Theta^2}{\eta} + \eta cd_\theta L^2 T
\end{align*}
\end{restatable}

\begin{proof}[Proof of \Cref{lemma:expected-suboptimality-under-smoothing}]
Observe that
\begin{align*}
    \E_T \left[\sum_{t=1}^T \widehat{\PR}(\theta_t)\right]
        - \sum_{t=1}^T \widehat{\PR}(\theta)
    &= \sum_{t=1}^T \E_t \left[
            \widehat{\PR}(\theta_t) - \widehat{\PR}(\theta)
        \right] \\
    &\leq \sum_{t=1}^T \E_t \left[
            \nabla \widehat{\PR}(\theta_t)^\top (\theta_t - \theta)
        \right]
        \tag{convexity of $\widehat{\PR}$} \\
    &= \sum_{t=1}^T \E_t \left[
            g_t^\top (\theta_t - \theta)
        \right]
        \tag{\Cref{claim:gradient-estimate-is-unbiased-and-bounded}}
\end{align*}
To decompose $g_t^\top (\theta_t - \theta)$, note that
\begin{align*}
    \|\theta_{t+1} - \theta\|^2
    &= \|\Pi_{(1-\delta)\theta} (\theta_t - \eta g_t) - x\|^2 \\
    &\leq \|\theta_t - \eta g_t - \theta\|^2 \\
    &= \|\theta_t - \theta\|^2 + \eta^2 \|g_t\|^2 - 2\eta \cdot g_t^\top (\theta_t - \theta)
\end{align*}
Therefore
\begin{align*}
    g_t^\top(\theta_t - x)
    &\leq \frac{\|\theta_t - \theta\|^2 - \|\theta_{t+1} - \theta\|^2 + \eta^2 \|g_t\|^2}{2\eta} \\
    \sum_{t=1}^T \E_t \left[
            g_t^\top (\theta_t - \theta)
        \right]
    &\leq \sum_{t=1}^T \E_t \left[
            \frac{\|\theta_t - \theta\|^2 - \|\theta_{t+1} - \theta\|^2 + \eta^2 \|g_t\|^2}{2\eta}
        \right] \\
    &\leq \frac{1}{2\eta} \E_t \left[
            \|\theta_1 - \theta\|^2 + \eta^2 cd_\Theta L^2T
        \right]
        \tag{Claim \ref{claim:gradient-estimate-is-unbiased-and-bounded}} \\
    &\leq \frac{D_\Theta^2}{2\eta} + \frac{\eta cd_\Theta L^2 T}{2}
        \tag{diameter of $\Theta$}
\end{align*}
as required.
\end{proof}

\begin{claim}[Regret from projection]
\label{claim:regret-from-projection}
For any $\theta \in \Theta$, $\PR(\theta_\delta) - \PR(\theta) \leq \delta D_\Theta L$.
\end{claim}
\begin{proof}
Since $\PR$ is $L$-Lipschitz and $\Pi_{(1-\delta)\Theta}$ projects from a set of diameter $D_\Theta$ to a set of diameter $(1-\delta)D_\Theta$, we have $\PR(\theta_\delta) - \PR(\theta) \leq L\|\theta_\delta - \theta\| \leq \delta D_\Theta L$.
\end{proof}

\begin{claim}[Optimality of projected parameters]Since $\PR$ is convex in $\theta$,
$\PR\left(\Pi_{(1-\delta)\Theta}(\theta_{\OPT})\right) = \argmin_{\theta \in (1-\delta)\Theta} \PR(\theta)$.
\end{claim}

\paragraph{Overall Regret Analysis for \Cref{lemma:regret-bound-for-model-convex-performative-risk}}

We can now complete our regret bound for \Cref{lemma:regret-bound-for-model-convex-performative-risk}. Recall the lemma statement:

\convexregretbound*

\begin{proof}[Proof of \Cref{lemma:regret-bound-for-model-convex-performative-risk}]
We have
\begin{align*}
    \calR_T(\calA_{\ref{algorithm:minimize-convex-function}},f)
    &= \sum_{t=1}^T \left[
            \estimatePR(\theta_t^+)
            + \estimatePR(\theta_t^-)
            - 2 \PR(\theta_{\OPT})
        \right] \\
    &= \underbrace{\sum_{t=1}^T \left[
            \widetilde{\PR}(\theta_t^+)
            + \widetilde{\PR}(\theta_t^-)
            - \PR(\theta_t^+)
            - \PR(\theta_t^-)
        \right]}_{\text{(I)}}
        + \underbrace{\sum_{t=1}^T \left[
                \PR(\theta_t^+)
                + \PR(\theta_t^-)
                - 2\widehat{\PR}(\theta_t)
            \right]}_{\text{(II)}} \\
      &\qquad  + \underbrace{2 \sum_{t=1}^T \left[
                \widehat{\PR}(\theta_t) - \E_t [\widehat{\PR}(\theta_t)]
            \right]}_{\text{(III)}} 
        + \underbrace{2 \sum_{t=1}^T \left[
                \E_t [\widehat{\PR}(\theta_t)] - \fhat(\theta^*_\delta)
            \right]}_{\text{(IV)}}\\
      &\qquad  + \underbrace{2 \sum_{t=1}^T \left[
                \widehat{\PR}(\theta^*_\delta) - f(\theta^*_\delta)
            \right]}_{\text{(V)}}
        + \underbrace{2 \sum_{t=1}^T \left[
                \PR(\theta^*_\delta) - \PR(\theta_\OPT)
            \right]}_{\text{(VI)}} \\
    &\leq \underbrace{
                2F\sqrt{T \log\frac{1}{p_1}}
            }_{\substack{
                \text{(I), w.p. $1-2p_1$} \\
                \text{(\Cref{claim:regret-from-estimating-f})}
            }}
        + \underbrace{
                4\delta LT
            }_{\substack{
                \text{(II), w.p. $1$} \\
                \text{(\Cref{claim:regret-from-smoothing-over-sphere-or-ball})}
            }}
        + \underbrace{
                2F\sqrt{T \log\frac{1}{p_2}}
            }_{\substack{
                \text{(III), w.p. $1-2p_2$} \\
                \text{(\Cref{claim:deviation-of-smoothed-function})}
            }}
        + \underbrace{
                \frac{2 D_\Theta^2}{\eta} + 2\eta cd_\theta L^2 T
            }_{\substack{
                \text{(IV), w.p. $1$} \\
                \text{(\Cref{lemma:expected-suboptimality-under-smoothing})}
            }}
        + \underbrace{
                2\delta LT
            }_{\substack{
                \text{(V), w.p. $1$} \\
                \text{(\Cref{claim:regret-from-smoothing-over-sphere-or-ball})}
            }}
        + \underbrace{
                2\delta D_\Theta LT
            }_{\substack{
                \text{(V), w.p. $1$} \\
                \text{(\Cref{claim:regret-from-projection})}
            }}
\end{align*}

Thus for any $p>0$, a choice of $p_1 = p_2 = p/4$, along with $\eta = 1/\sqrt{d_\theta T}$ and any $\delta \leq \sqrt{d_\theta/T}$, yields $\calR_T(\calA_{\ref{algorithm:minimize-convex-function}},\PR) = O(\sqrt{d_\theta T\log \frac{1}{p}})$ with probability at least $1-p$. 
Finally, since $\estimatePR$ is queried twice per step, $n = 2T$, which gives us $\calR_n(\calA_{\ref{algorithm:minimize-convex-function}},\PR) = \calR_T(\calA_{\ref{algorithm:minimize-convex-function}},\PR) = O(\sqrt{d_\theta n\log \frac{1}{p}})$,
completing the proof.
\end{proof}

\section{Missing Proofs and Additional Examples for \Cref{subsec: overview}}
\label{sec:more-example}

\paragraph{Derivations for \Cref{example:mixture-dominance-too-strong-condition}}

Since $\phi$ is strictly increasing in $[0, 1]$, the inverse mapping $\phi^{-1}$ is well-defined, and we can reformulate the performative risk $\PR(\theta)$ as a function of $\varphi_\theta$, denoted $\PRdagger(\varphi_\theta)$, as follows:
\begin{align*}
    \PR(\theta;x)
    &= \E_{y \sim \cc{Bern}(\varphi_\theta)}[\ell(\theta;x,y)] \\
    &= \varphi_\theta \ell(\theta;x,1) + (1- \varphi_\theta) \ell(\theta;x,0) \\
    &= \varphi_\theta \ell\left(\phi^{-1}(\varphi_\theta);x,1\right) + (1- \varphi_\theta) \ell\left(\phi^{-1}(\varphi_\theta);x,0\right) \\
    &=: \PRdagger(\varphi_\theta;x)
\end{align*}

Plugging in $\ell$, we have
\begin{align*}
    \PRdagger(\varphi_\theta;x)
    &= -\varphi_\theta \cdot \left(\phi^{-1}(\varphi_\theta)x - 1\right)^2
        - (1- \varphi_\theta) \cdot \left(\phi^{-1}(\varphi_\theta)x\right)^2 \\
    &= -\varphi_\theta \cdot \left(\sqrt{\varphi_\theta}x - 1\right)^2
        - (1- \varphi_\theta)\varphi_\theta x^2
        \tag{$\phi^{-1}(\varphi_\theta) = \sqrt{\varphi_\theta}$}
\end{align*}
Note that for all $x \in [0,1]$, $\PRdagger(\varphi_\theta;x) = \PR(\theta;x)$ is convex in $\varphi_\theta$ over $[0,1]$. In contrast,
\begin{align}
    \PR(\theta;x)
    &= \theta^2 \cdot \ell(\theta;x,1) + (1-\theta^2) \cdot \ell(\theta;x,0) \\
    &= -\theta^2\cdot (\theta x - 1)^2 - (1 - \theta^2) \cdot (\theta x)^2
\end{align}
which is non-convex in $\theta$ over $[0,1]$ for all $x \in [0,1]$.

\begin{figure}[h!]
    \begin{subfigure}
        \centering
        \includegraphics[width=0.25\linewidth]{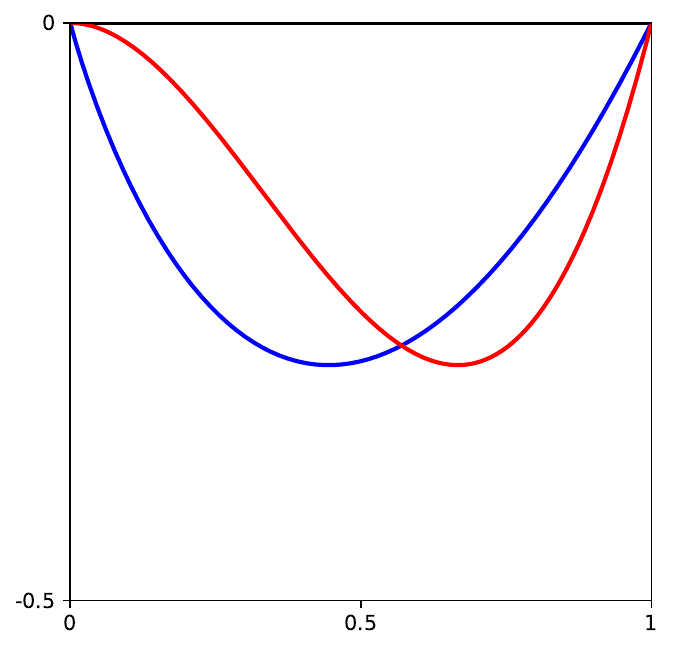}
    \end{subfigure}
    \begin{subfigure}
        \centering
        \includegraphics[width=0.65\linewidth]{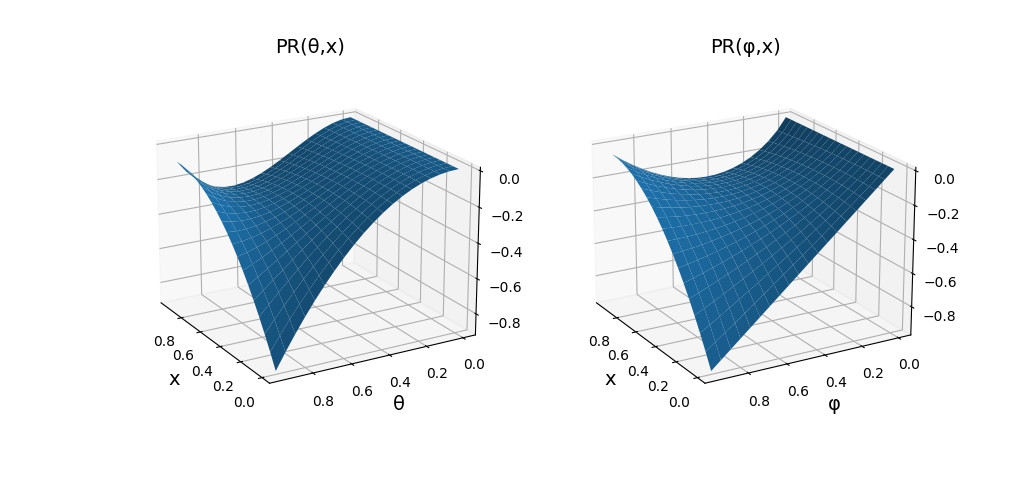}
    \end{subfigure}
    \caption{An example showing that our assumption is weaker than the mixture dominance assumption in \citet{miller2021outside}. In the left figure, the blue curve represents the function $\PRdagger(\varphi_\theta)$ which is convex w.r.t the data distribution parameter $\varphi_\theta$; while the red curve represents the function $\PR(\theta)$, which is not a convex function with respect to $\theta$. In the right two figures, we compare $\PR$ as a function of the model parameter $\theta$ and as a function of the distribution parameter $\phi$.}
\end{figure}

Notice that \Cref{example:mixture-dominance-too-strong-condition} can be generalized to any distribution map $\phi(\theta)$ that satisfies $\phi(\theta) = \theta^\alpha$ for any $\alpha> 1$, and any $\ell_\beta$ loss for even $\beta$ value. Below in \Cref{fig:another-example}, we provide the plot for for $\phi(\theta) = \theta^4$ with $\ell_4$ norm loss ($L_4$ norm is defined as $L_4(x, y)=\left(\sum_{i=1}^d\left|x_i-y_i\right|^4\right)^{\frac{1}{4}}$ where $d$ is the dimension of $x$ and $y$). The original $\PR$ loss $\PR(\theta)$ is in red, which is non-convex), and the reformulated PR loss $\PR^\dagger (\theta)$ is in blue via reparameterization, which is convex).

\begin{figure}[h]
    \centering
    \includegraphics[width=0.35\linewidth]{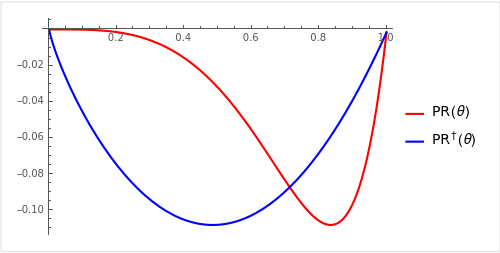}
    \caption{Another example showing $\PR$ is convex in $\phi$ but not $\theta$. The original $\PR$ loss $\PR(\theta)$ is in red, which is non-convex), and the reformulated PR loss $\PR^\dagger (\theta)$ is in blue via reparameterization, which is convex)}
    \label{fig:another-example}
\end{figure}

In addition, \Cref{example:mixture-dominance-too-strong-condition} can also be a generalized monotone polynomial function $\phi(\theta)$. For example, $\phi(\theta) = c \theta^2 + d \theta$, for $d\geq 0, d \geq -2c$.

\paragraph{Derivations for \cref{example:gaussian}}

\label{proof:example-gaussian}
    \begin{align*}
        \PR(\theta) &= \E_{x\sim \calD(\theta)} \ell(x;\theta)\\
        &= \E_{x\sim N(\varphi(\theta), \sigma^2)} (\theta x)^2\\
        &= \theta^2 (\sigma^2 + \varphi(\theta)^2)
    \end{align*}
when $\varphi(\theta) = \sqrt{\theta}$, $\PR(\theta)$ is not convex in $\theta$. To see this:
    \begin{align*}
        \PR(\theta) = \theta^2(\sigma^2 + \theta)
    \end{align*}
On the other hand, since $\phi = \varphi(\theta) = \sqrt{\theta}$, we have
\begin{align*}
    \PR(\phi) = \phi (\theta^2 + \phi)
\end{align*}
which is convex in $\phi$.

\paragraph{Derivations for \cref{example:uniform}}

\label{proof:example-uniform}
    \begin{align*}
        \PR(\theta) &= \E_{y\sim \calD(\theta)} (\ell(\theta; x, y))\\
        &= - \E_{\text{Uniform}[0, \varphi(\theta)]} (\theta x - y)^2\\
        &= \theta^2 x^2 - \theta x \varphi(\theta) + \frac{1}{3} \varphi(\theta)^2
    \end{align*}
when $\varphi(\theta) = \theta^2$, we have
    \begin{align*}
        \PR(\theta) = \theta^2 x^2 - \theta^3 x + \frac{1}{3}\theta^4
    \end{align*}
    which is non-convex in $\theta$.
On the other hand, denote $\phi = \varphi(\theta)$, we have $\theta = \sqrt{\phi}$, plug it into $\PR$, we have:
\begin{align*}
    \PR(\phi) = \phi x^2 - \phi \sqrt{\phi} + \frac{1}{3} \phi^2
\end{align*}
which is convex in $\phi$.

\section{Omitted Proof for \Cref{sec:outer-algorithm}}
\label{proof:outer-algorithm}

We present a series of lemmas and claims that are helpful for proving \Cref{theorem:regret-bound-for-indirectly-convex-functions}.

\begin{claim}[Deviation of $\PR^\dagger$ due to error of $\learnmodel$]
\label{claim:deviation-of-PRdagger-due-to-error-of-learnmodel}
If $\PR^\dagger$ is $L^\dagger$-Lipschitz, then for any $\phi \in \Phi$, the value $\hat{\theta} \in \Theta$ returned by $\learnmodel(\phi,\epsilon_\LM,p_\LM)$ satisfies $|\PR^\dagger(\phi) - \PR(\hat{\theta})| \leq L^\dagger \epsilon_\LM$ with probability at least $1-p_\LM$.
\end{claim}

\begin{proof}
\label{proof:deviation-of-PRdagger-due-to-error-of-learnmodel}
We have
\begin{align*}
    \left|\PR^\dagger(\phi) - \PR(\hat{\theta})\right|
    &= \left|\PR^\dagger(\phi) - \PR^\dagger(\varphi(\hat{\theta}))\right| \\
    &\leq L^\dagger \left\|\phi - \varphi(\hat{\theta})\right\|
        \tag{Lipschitzness of $\PR^\dagger$} \\
    &\leq L^\dagger \epsilon_\LM
        \tag{guarantee of $\learnmodel$}
\end{align*}
where the last inequality holds with probability at least $1-p_\LM$.
\end{proof}

\begin{claim}[Deviation of gradient estimate due to error of $\learnmodel$ and $\estimatePR$]
\label{claim:deviation-of-gradient-estimate-due-to-error-of-learnmodel-and-estimatePR}
Define
\begin{align}
    \tilde{g}_t
        := \frac{d_\Phi}{\delta} \widetilde{\PR}(\hat{\theta}_t^+) u_t
    \qquad \text{and} \qquad
    g_t
        := \frac{d_\Phi}{\delta} \PR^\dagger(\phi_t^+) u_t
\end{align}
For any $t \in [T]$,
\begin{align*}
    g_t - \tilde{g}_t
    \leq \frac{d_\Phi}{\delta} \left[
            \PR(\hat{\theta}_t^+)
            - \widetilde{\PR}(\hat{\theta}_t^+)
            + \PR^\dagger(\phi_t^+)
            - \PR(\hat{\theta}_t^+)
        \right]
        u_t .
\end{align*}
\end{claim}
\begin{proof}
\label{proof:deviation-of-gradient-estimate-due-to-error-of-learnmodel-and-estimatePR}
We have
\begin{align*}
    g_t
    &= \frac{d_\Phi}{\delta} \PR^\dagger(\phi_t^+) u_t \\
    &= \frac{d_\Phi}{\delta} \left[
            \widetilde{\PR}(\hat{\theta}_t^+)
            - \widetilde{\PR}(\hat{\theta}_t^+)
            + \PR(\hat{\theta}_t^+)
            - \PR(\hat{\theta}_t^+)
            + \PR^\dagger(\phi_t^+)
        \right] u_t \\
    &= \tilde{g}_t
        + \frac{d_\Phi}{\delta} \left[
                \PR(\hat{\theta}_t^+)
                - \widetilde{\PR}(\hat{\theta}_t^+)
                + \PR^\dagger(\phi_t^+)
                - \PR(\hat{\theta}_t^+)
            \right] u_t
            \tag{definition of $\tilde{g}_t$}
\end{align*}
\end{proof}


\begin{restatable}[Expected suboptimality under smoothing for $\PRdagger$]{lem}{}
\label{lemma:high-probability-expected-suboptimality-under-smoothing}
For any $\phi \in \Phi$, with probability at least $1 - Tp_\LM$ over the calls to $\learnmodel$,
\begin{align*}
    \E_T \left[
            \sum_{t=1}^T \PRhat^\dagger(\phi_t)
        \right]
        - \sum_{t=1}^T \PRhat^\dagger(\phi)
    \leq \frac{D_\Phi^2}{\eta}
        + \eta cd_\Phi L^2 T
        + \frac{D_\Phi L^\dagger \epsilon_\LM d_\Phi T}{\delta}
\end{align*}
\end{restatable}

\begin{proof}[Proof of \Cref{lemma:high-probability-expected-suboptimality-under-smoothing}]
For any $\phi \in \Phi$, we have
\begin{align*}
    & \E\left[
            \sum_{t=1}^T \PRhat^\dagger(\phi_t)
        \right]
        - \sum_{t=1}^T \PRhat^\dagger(\phi) \\
    =& \sum_{t=1}^T \E\left[
            \PRhat^\dagger(\phi_t) - \PRhat^\dagger(\phi)
        \right] \\
    \leq& \sum_{t=1}^T \E\left[
            \nabla \PRhat^\dagger(\phi_t)^\top(\phi_t - \phi)
        \right]
        \tag{convexity of $\PRhat^\dagger$}\\
    =& \sum_{t=1}^T \E\left[
            g_t^\top (\phi_t - \phi)
        \right]
        \tag{\Cref{claim:gradient-estimate-is-unbiased-and-bounded}} \\
    =& \sum_{t=1}^T \E\left[
        \left(
            \tilde{g}_t + \frac{d_Y}{\delta} \left[
                \PR(\hat{\theta}_t^+)
                - \PRtilde(\hat{\theta}_t^+)
                + \PRdagger(\phi_t^+)
                - \PR(\hat{\theta}_t^+)
            \right] \cdot u_t
        \right)^\top(\phi_t - \phi)
    \right]
        \tag{\Cref{claim:deviation-of-gradient-estimate-due-to-error-of-learnmodel-and-estimatePR}} \\
    =& \sum_{t=1}^T \E\left[
        \left(
            \tilde{g}_t + \frac{d_Y}{\delta} \left[
                \PRdagger(\phi_t^+)
                - \PR(\hat{\theta}_t^+)
            \right] \cdot u_t
        \right)^\top(\phi_t - \phi)
    \right]
        \tag{$\E[\PRtilde(\cdot)] = \PR(\cdot)$ since $\estimatePR$ is unbiased} \\
    =& \sum_{t=1}^T \E\left[
            \tilde{g}_t^\top (\phi_t - \phi)
        \right]
        + \frac{d_Y}{\delta} \sum_{t=1}^T \E\left[
            \left(
                \PRdagger(\phi_t^+)
                - \PR(\hat{\theta}_t^+)
            \right)
            u_t^\top
            (\phi_t - \phi)
        \right] \\
    \leq& \sum_{t=1}^T \E\left[
            \tilde{g}_t^\top (\phi_t - \phi)
        \right]
        + \frac{d_Y}{\delta} \sum_{t=1}^T \E\left[
            \left|
                \PRdagger(\phi_t^+)
                - \PR(\hat{\theta}_t^+)
            \right|
            \cdot \|u_t\|
            \cdot \|\phi_t - \phi\|
        \right] \\
    \leq& \sum_{t=1}^T \E\left[
            \tilde{g}_t^\top (\phi_t - \phi)
        \right]
        + \frac{d_Y}{\delta} \sum_{t=1}^T \E\left[
            L^\dagger \epsilon_h
            \cdot D_Y
        \right]
        \tag{\Cref{claim:deviation-of-PRdagger-due-to-error-of-learnmodel}, w.p. $1-Tp_h$} \\
    \leq& \frac{D_Y^2}{\eta}
        + \eta cd_Y L^2 T
        + \frac{d_Y}{\delta}L^\dagger \epsilon_h D_Y T
        \tag{same argument as in \Cref{lemma:expected-suboptimality-under-smoothing}}
\end{align*}
\end{proof}

\paragraph{Regret analysis for the outer algorithm in total number of step $T$}
We can now complete our regret bound for $\minimizePR$ (\Cref{algorithm:minimize-indirectly-convex-function}). We recall the theorem statement for  \Cref{theorem:regret-bound-for-indirectly-convex-functions}:

\outeralgorithmregret*

\begin{proof}[Proof of \Cref{theorem:regret-bound-for-indirectly-convex-functions}]
\label{proof:regret-bound-for-indirectly-convex-functions}
We have
\begin{align*}
    & \ \calR_T(\minimizePR,\PR) \\
    &= \sum_{t=1}^T \left[
            \estimatePR(\hat{\theta}_t^+)
            + \estimatePR(\hat{\theta}_t^-)
            - 2 \PR(\theta_\OPT)
        \right] \\
    &= \underbrace{\sum_{t=1}^T \left[
            \PRtilde(\hat{\theta}_t^+)
            + \PRtilde(\hat{\theta}_t^-)
            - \PR(\hat{\theta}_t^+)
            - \PR(\hat{\theta}_t^-)
        \right]}_{\text{(I)}}
        + \underbrace{\sum_{t=1}^T \left[
                \PR(\hat{\theta}_t^+)
                + \PR(\hat{\theta}_t^-)
                - \PRdagger(\phi_t^+)
                - \PRdagger(\phi_t^-)
            \right]}_{\text{(II)}} \\
        &\qquad
        + \underbrace{\sum_{t=1}^T \left[
                \PRdagger(\phi_t^+)
                + \PRdagger(\phi_t^-)
                - 2\PRhat^\dagger(\phi_t)
            \right]}_{\text{(III)}}
        + \underbrace{2 \sum_{t=1}^T \left[
                \PRhat^\dagger(\phi_t)
                - \E_t [\PRhat^\dagger(\phi_t)]
            \right]}_{\text{(IV)}} \\
        &\qquad
        + \underbrace{2 \sum_{t=1}^T \left[
                \E_t [\PRhat^\dagger(\phi_t)]
                - \PRhat^\dagger(\phi^*_\delta)
            \right]}_{\text{(V)}}
        + \underbrace{2 \sum_{t=1}^T \left[
                \PRhat^\dagger(\phi^*_\delta)
                - \PRdagger(\phi^*_\delta)
            \right]}_{\text{(VI)}}
        + \underbrace{2 \sum_{t=1}^T \left[
                \PRdagger(\phi^*_\delta)
                - \PRdagger(\phi_\OPT)
            \right]}_{\text{(VII)}} \\
    &\leq \underbrace{
                2F\sqrt{T \log\frac{1}{p_1}}
            }_{\substack{
                \text{(I), w.p. $1-2p_1$} \\
                \text{(\Cref{claim:regret-from-estimating-f})}
            }}
        \qquad + \underbrace{
                2L^\dagger \epsilon_\LM T
            }_{\substack{
                \text{(II), w.p. $1-2Tp_\LM$} \\
                \text{(\Cref{claim:deviation-of-PRdagger-due-to-error-of-learnmodel})}
            }}
        \qquad + \underbrace{
                4\delta LT
            }_{\substack{
                \text{(III), w.p. $1$} \\
                \text{(\Cref{claim:regret-from-smoothing-over-sphere-or-ball})}
            }}
        \qquad + \underbrace{
                2F\sqrt{T \log\frac{1}{p_2}}
            }_{\substack{
                \text{(IV), w.p. $1-2p_2$} \\
                \text{(\Cref{claim:deviation-of-smoothed-function})}
            }} \\
        &\qquad \qquad \qquad
        + \underbrace{
                \frac{2\distributionDiameter^2}{\eta}
                + 2 \eta c\distributionDiameter L^2 T
                + \frac{2 \distributionDiameter L^\dagger \epsilon_\LM \distributionDiameter T}{\delta}
            }_{\substack{
                \text{(V), w.p. $1-2Tp_\LM$} \\
                \text{(\Cref{lemma:high-probability-expected-suboptimality-under-smoothing})}
            }}
        \qquad + \underbrace{
                2\delta L^\dagger T
            }_{\substack{
                \text{(VI), w.p. $1$} \\
                \text{(\Cref{claim:regret-from-smoothing-over-sphere-or-ball})}
            }}
        \qquad + \underbrace{
                2\delta \distributionDiameter L^\dagger T
            }_{\substack{
                \text{(VII), w.p. $1$} \\
                \text{(\Cref{claim:regret-from-projection})}
            }}
\end{align*}

Recall that in \Cref{algorithm:minimize-indirectly-convex-function}, we set $\delta = \sqrt{\epsilon_\LM \distributionDiameter}$ and $\eta = 1/\sqrt{\distributionDiameter T}$. Thus for any $p'>0$, a choice of $p_1 = p_2 = p'/4$ yields
\begin{align*}
    \calR_T(\calA_{\ref{algorithm:minimize-indirectly-convex-function}},\PR)
    = O\left(
        \sqrt{\distributionDiameter T}
        + \sqrt{\epsilon_\LM \distributionDiameter} \cdot T
        + \sqrt{T \log\frac{1}{p'}}
    \right)
\end{align*}
with probability at least $1 - p' - 2Tp_\LM$ as required.
\end{proof}

\section{Omitted Proof for \Cref{sec:learnmodel}}
\label{sec: omitted-proof-learnmodel}
We first provide a proof for \Cref{lemma:lip-phi-in-KL}. Recall the lemma statement:

\kllipschitzcondition*

\begin{proof}[Proof of \Cref{lemma:lip-phi-in-KL}]
\begin{align*}
  & \left|
        \KL(\phi||\varphi(\theta_1))-\KL(\phi||\varphi(\theta_2))
    \right|\\
    =& \left|
        \int_z p(z|\phi) \log \frac{p(z|\phi)}{p(z|\varphi(\theta_1))}dz - 
        \int_z p(z|\phi) \log \frac{p(z|\phi)}{p(z|\varphi(\theta_2))}dz
    \right|\\
    =& \left|
        \int_z p(z|\phi) (\log p(z|\varphi(\theta_1)) - \log p(z|\varphi(\theta_2)))dz 
    \right|\\
    \leq & 
        \int_z p(z|\phi) \left|
            \log p(z|\varphi(\theta_1)) - \log p(z|\varphi(\theta_2))
        \right|dz \\
    \leq &  \int_z p(z|\phi) L_\KL \|\theta_1 - \theta_2\| dz \tag{$\mathcal{P}_{1}$ and $\mathcal{P}_{2}$ are lipschitzness continuous, Theorem 3 of \cite{Honorio2012lipschitz}}\\
    = & L_\KL \|\theta_1 - \theta_2\| \underbrace{\int_z p(z|\phi) dz}_{=1}\\
    =& L_\KL \|\theta_1 - \theta_2\|
\end{align*}
\end{proof}

Next, we provide the proof for \Cref{lemma:phi-bound-by-KL}. Recall the lemma statement: 
\phiboundbykl*
\begin{proof}[Proof of \Cref{lemma:phi-bound-by-KL}]
    \begin{align*}
        \|\phi_1 - \phi_2\|_2 
    \leq L_\TV d_\TV (\phi_1, \phi_2) 
    \leq L_\TV \sqrt{\frac{1}{2} \KL(\phi_1, \phi_2)}
    \triangleq L_{\phi} \sqrt{\KL(\phi_1, \phi_2)}
    \end{align*}   
The second inequality is due to Pinsker's inequality.
\end{proof}

We then show the example provide by \Cref{example:convex-kl} is convex in $\theta$. Recall the example:
\klconvexexample*

Below we provide proof for it being convex in $\theta$:
\begin{proof}[Proof for \Cref{example:convex-kl} being convex in $\theta$]
Under condition 1, we have $p(z;\phi) =  \frac{1}{\exp(c\varphi(\theta))}$. We can rewrite the $\KL(\phi||\varphi(\theta))$ divergence as:
\begin{align*}
    \KL(\phi||\varphi(\theta)) =& \int_z p(z;\phi)\log \frac{p(z;\phi)}{p(z;\varphi(\theta))} dz\\
    =& \int_z \frac{1}{\exp(c\phi)} \log \frac{\exp(c\varphi(\theta))}{ \exp(c\phi)} dz\\
    =& \frac{\exp(c\varphi(\theta))}{\exp(c\phi)} \log \frac{\exp(c\varphi(\theta)}{\exp(c\phi)}\\
    =& \exp(c(\varphi(\theta) - \phi)) c(\varphi(\theta) - \phi)
\end{align*}

Denote $\KL(\phi||\varphi(\theta)) = f(g(\theta))$ where $f(x) = cx\exp(cx)$ and $g(\theta) = \varphi(\theta) - \phi$.

To show \Cref{eqn:kl-divergence} is convex in $\theta$, it suffices to show f(x) is convex non-decreasing in x, and $g(\theta)$ is convex in $\theta$. 
First, $g(\theta)$ is convex in $\theta$ due to condition 2. \\
For $f(x)$, take the first and second derivative and find conditions to make them both non negative:
\begin{align*}
    \frac{\partial f(x)}{\partial x} &= c\exp(cx)+ cx^2 \exp(cx)\\
    &= c\exp(cx)(1 + cx)\geq 0\\
    \frac{\partial^2 f(x)}{\partial x^2} &= c^2\exp(cx)(2 + cx)\geq 0
\end{align*}

It suffices to set $(2+cx)\geq 0$ and $c(1+cx)\geq 0$ which suffices to set $c\geq \frac{2}{\max |\varphi(\theta) - \phi|}$.

\end{proof}

\paragraph{Regret Analysis and convergence guarantee of $\learnmodel$ in total number of steps $S$} We can now complete our regret bound for $\learnmodel$ (\Cref{algorithm:learnmodel}). Recall the theorem statement:
\learnmodelregret*

\begin{proof} [Proof of \Cref{theorem:regret-bound-for-learnmodel}]
\label{proof:regret-bound-for-learnmodel}
\begin{align*}
    &\calR_S(\learnmodel, \KL) \\
    =& \sum_{s=1}^S 
    \left[
        \widetilde{\KL}(\phi||\varphi(\theta_{s}^+)) + \widetilde{\KL}(\phi||\varphi(\theta_{s}^-)) - 
        2\underbrace{\KL(\phi||\varphi(\vartheta^*(\phi)))}_{=0, \varphi(\vartheta^*(\phi))) = \phi }
    \right]\\
    =& \underbrace{\sum_{s=1}^S 
    \left[
        \widetilde{\KL}(\phi||\varphi(\theta_{s}^+)) - {\KL}(\phi||\varphi(\theta_{s}^+)) + \widetilde{\KL}(\phi||\varphi(\theta_{s}^+)) 
        - \KL(\phi||\varphi(\theta_{s}^-))
    \right]}_{\text{(I)}}\\
     &\qquad + \underbrace{\sum_{s=1}^S 
     \left[
            \KL(\phi||\varphi(\theta_s^+))
                + \KL(\phi||\varphi(\theta_s^-))
                - 2\widehat{\KL}(\phi||\varphi(\theta_s))
            \right]}_{\text{(II)}} \\
        &\qquad
        + \underbrace{2 \sum_{s=1}^S \left[
                \widehat{\KL}(\phi||\varphi(\theta_s))
                - \E_s [\widehat{\KL}(\phi||\varphi(\theta_s))]
            \right]}_{\text{(III)}}
        + \underbrace{2 \sum_{s=1}^S \left[
                \E_s [\widehat{\KL}(\phi||\varphi(\theta_s))]
                - \widehat{\KL}(\phi||\varphi(\theta^*_\delta))
            \right]}_{\text{(IV)}}\\
        &\qquad
        + \underbrace{2 \sum_{s=1}^S \left[         \widehat{\KL}(\phi||\varphi(\theta^*_\delta))
                - {\KL}(\phi||\varphi(\theta^*_\delta))
            \right]}_{\text{(V)}}
        + \underbrace{2 \sum_{s=1}^S \left[
                {\KL}(\phi||\varphi(\theta^*_\delta))
                - {\KL}(\phi||\varphi(\theta^*))
            \right]}_{\text{(VI)}} \\
            &\leq \underbrace{
                2\epsilon_\KL S
            }_{\substack{
                \text{(I), w.p. $1-2 S p_\KL$} \\
                \text{(Assumption \ref{ass:kl-oracle})}
            }}
        \qquad + \underbrace{
                4\delta L_\KL S
            }_{\substack{
                \text{(II), w.p. $1$} \\
                \text{(\Cref{claim:regret-from-smoothing-over-sphere-or-ball})}
            }}
        \qquad + \underbrace{
                2F_\KL \sqrt{S \log\frac{1}{p_2}}
            }_{\substack{
                \text{(III), w.p. $1-2p_2$} \\
                \text{(\Cref{claim:deviation-of-smoothed-function})}
            }} \\
        &\qquad 
        + \underbrace{
                \frac{2D_\Theta^2}{\eta_\LM}
                + 2 \eta_\LM d_\Theta L_\KL^2 S
                + \frac{2 D_\Theta L_\KL \epsilon_\KL d_\Theta S}{\delta_\LM}
            }_{\substack{
                \text{(IV), w.p. $1-2 S p_\KL$} \\
                \text{(Similar argument as \Cref{lemma:high-probability-expected-suboptimality-under-smoothing})}
            }}
        \qquad + \underbrace{
                2\delta_\LM L_\KL S
            }_{\substack{
                \text{(V), w.p. $1$} \\
                \text{(\Cref{claim:regret-from-smoothing-over-sphere-or-ball})}
            }}
        \qquad + \underbrace{
                2\delta_{\LM} D_\Theta L_\KL S
            }_{\substack{
                \text{(VI), w.p. $1$} \\
                \text{(\Cref{claim:regret-from-projection})}
            }}
\end{align*}
Similar to \Cref{algorithm:minimize-indirectly-convex-function}, we set $\delta_\LM = \sqrt{\epsilon_\KL d_\Theta} $, $\eta_\LM = 1/\sqrt{d_\Theta S}$. For any $p_2 = p'/2 >0$, it yields
\begin{align*}
    R_S(\learnmodel, \KL)
    = O\left(
            \sqrt{d_\Theta S}
            + \sqrt{\epsilon_\KL d_\theta}S
            + \sqrt{S \log\frac{1}{p}}
        \right)
\end{align*}
with probability $1 - p' - 2S p_\KL$ > 0.

\end{proof}

\section{Omitted Proof for \Cref{sec:putting-things-together}}
\label{sec:proof-put-things-together}
We start with leveraging \Cref{theorem:regret-bound-for-indirectly-convex-functions}
to show the following convergence guarantee for $\minimizePR$ (\Cref{algorithm:minimize-indirectly-convex-function}).

\begin{claim}[Convergence of $\minimizePR$]
\label{claim:convergence-of-minimizePR}
Given any $\epsilon, p > 0$, $\minimizePR$ outputs an $\epsilon$-suboptimal solution for $\PR(\theta)$ with probability at least $1-p$. Moreover, $\minimizePR$ runs for $T = O(d_\Phi/\epsilon^2)$ steps and performs $O(d_\Phi/\epsilon^2)$ queries to $\estimatePR$, as well as $O(d_\Phi/\epsilon^2)$ queries to $\learnmodel$ with $\epsilon_\LM = O(\epsilon^2)$ and $p_\LM = O(\epsilon^2 p/d_\Phi)$.
\end{claim}
\begin{proof}[Proof of \Cref{claim:convergence-of-minimizePR}]

Choosing $\epsilon_\LM = 1/T$, $p_\LM = p/2T$, and $p' = p/2$, \Cref{theorem:regret-bound-for-indirectly-convex-functions} shows that $\minimizePR$ satisfies
\begin{align*}
    \calR_T(\minimizePR,\PR)
    = O\left(\sqrt{d_\Phi T}\right)
\end{align*}
with probability $1 - p$, using $2T$ queries to $\estimatePR$ and $2T$ queries to $\learnmodel$. By \Cref{proposition:sublinear-regret-implies-convergence}, $T = O(d_\Phi / \epsilon^2)$ steps suffice to output a model that is $\epsilon$-suboptimal with respect to $\PR$. Plugging in this bound on $T$ into the expressions for $\epsilon_\LM$ and $p_\LM$ above yields the result.
\end{proof}

Similarly, we have the convergence guarantee for $\learnmodel$ as well:

\begin{claim}[Convergence of $\learnmodel$]
    \label{claim:convergence-of-learnmodel-KL}
    Given any $\phi \in \Phi$ and $\epsilon_\LM, p_\LM> 0$, $\learnmodel$ outputs an $\epsilon_\LM$-suboptimal model for \Cref{eqn:kl-divergence}
    with probability at least $1 - p_\LM$. Moreover, $\learnmodel$ runs for $S = O(d_\Theta/\epsilon_\LM^2)$ steps and performs two queries to $\estimatekl$ per step with $N_\KL( \frac{\epsilon^2_\LM}{d_\theta},\frac{\epsilon^2_\LM p_\LM}{4d_\theta})$ samples per query.
\end{claim}

\begin{proof}[Proof of \Cref{claim:convergence-of-learnmodel-KL}]
Choosing $\epsilon_\KL = 1/S$, $p_\KL = p_\LM/4S$ and $p' = p_\LM/2$, \Cref{theorem:regret-bound-for-learnmodel} shows that $\learnmodel$ satisfies
\begin{align*}
     \calR_S(\learnmodel,\KL)
    = O\left(
        \sqrt{d_\Phi S}
    \right)
\end{align*}
By \Cref{proposition:sublinear-regret-implies-convergence}, $S = O(d_\Theta / \epsilon_\LM^2)$ steps suffice to output a model that is $\epsilon_\LM$-suboptimal with respect to $\KL$; thus we have $\epsilon_\KL = \frac{\epsilon_\LM^2}{d_\Theta}$, $p_\KL = \frac{1}{4Sp_\LM}$. In total, $\learnmodel$ makes $2S$ queries to $\estimatekl$ with $N_\KL(\frac{\epsilon_\LM^2}{d_\Theta}, \frac{\epsilon^2_\LM p_\LM}{4d_\theta})$ samples per query.
\end{proof}

Now are are ready to prove \Cref{theorem:total-regret}. Recall the theorem statement:

\totalregret*

\begin{proof}[Proof of \Cref{theorem:total-regret}]
\label{proof:total-regret}
Let $T$ be the number of steps executed by $\minimizePR$, and $S$ the number of steps in $\learnmodel$. Let $N_\KL(\epsilon_\KL, p_\KL)$ (or $N_\KL$ for short) denote the number of samples used by $\estimatekl(\cdot, \cdots, \epsilon_\KL, p_\KL)$. Since $\minimizePR$ calls $\estimatePR$ and $\learnmodel$ $2T$ times, and $\learnmodel$ calls $\estimatekl$ $2S$ times, the overall number of samples is $N = 2(2N_\KL S + 1)T$.

Let $\theta_{t,s}^+, \theta_{t,s}^-$ denote the models deployed by $\estimatekl$ in the $s$-th step of $\learnmodel$ within the $t$-th step of $\minimizePR$, obtaining samples $z_{t,s,1}^+,\ldots,z_{t,s,N_\KL}^+$ and $z_{t,s,1}^-,\ldots,z_{t,s,N_\KL}^-$, respectively. Similarly, let $\hat{\theta}_t^+, \hat{\theta}_t^-$ denote the models deployed by $\estimatePR$ in the $t$-th step of $\minimizePR$, obtaining samples $\hat{z}_t^+, \hat{z}_t^-$.

The total regret can be written as
\begin{align*}
    &\quad \calR_N(\minimizePR, \PR) \\
    &= \sum_{t=1}^T \left(
            \ell(\hat{z}_t^+; \hat{\theta}_t^+)
            + \ell(\hat{z}_t^-; \hat{\theta}_t^-)
            - 2\PR(\theta^*)
            + \sum_{s=1}^S \sum_{i=1}^{N_\KL} \left[
                \ell(z_{t,s,i}^+; \theta_{t,s}^+)
                + \ell(z_{t,s,i}^-; \theta_{t,s}^-)
                - 2\PR(\theta^*)
            \right]
        \right) \\
    &= \underbrace{
            \sum_{t=1}^T \left(
                \ell(\hat{z}_t^+; \hat{\theta}_t^+)
                - \PR(\hat{\theta}_t^+)
                + \ell(\hat{z}_t^-; \hat{\theta}_t^-)
                - \PR(\hat{\theta}_t^-)
                + \sum_{s=1}^S \sum_{i=1}^{N_\KL} \left[
                    \ell(z_{t,s,i}^+; \theta_{t,s}^+)
                    - \PR(\theta_{t,s}^+)
                    + \ell(z_{t,s,i}^-; \theta_{t,s}^-)
                    - \PR(\theta_{t,s}^-)
                \right]
            \right)
        }_{n \text{ difference terms with expectation zero}} \\
        &\qquad \qquad
        + \sum_{t=1}^T \left(
            \PR(\hat{\theta}_t^+)
            + \PR(\hat{\theta}_t^-)
            - 2\PR(\theta^*)
            + \sum_{s=1}^S \sum_{i=1}^{N_\KL} \left[
                \PR(\theta_{t,s}^+)
                + \PR(\theta_{t,s}^-)
                - 2\PR(\theta^*)
            \right]
        \right) \\
    &= O\left(\sqrt{N}\right)
        + \sum_{t=1}^T \left(
            \PR(\hat{\theta}_t^+)
            + \PR(\hat{\theta}_t^-)
            - 2\PR(\theta^*)
            + \sum_{s=1}^S \sum_{i=1}^{N_\KL} \left[
                \PR(\theta_{t,s}^+)
                + \PR(\theta_{t,s}^-)
                - 2\PR(\theta^*)
            \right]
        \right)
        \tag{by Hoeffding's inequality, w.p. $1-p'$} \\
    &= O\left(\sqrt{N}\right)
        + \sum_{t=1}^T \left[
                \PR(\hat{\theta}_t^+)
                + \PR(\hat{\theta}_t^-)
                - 2\PR(\theta^*)
            \right] \\
        &\qquad \qquad
        + N_\KL \cdot \sum_{t=1}^T \sum_{s=1}^S \left[
                (\PR(\theta_{t,s}^+) + \PR(\theta_{t,s}^-))
                - (\PR(\hat{\theta}_t^+) + \PR(\hat{\theta}_t^-))
                + (\PR(\hat{\theta}_t^+) + \PR(\hat{\theta}_t^-))
                - 2\PR(\theta^*)
            \right] \\
    &= O\left(\sqrt{N}\right)
        + (N_\KL S + 1) \sum_{t=1}^T \left[
            \PR(\hat{\theta}_t^+)
            + \PR(\hat{\theta}_t^-)
            - 2\PR(\theta^*)
        \right]\\
       &\qquad \qquad
       + N_\KL \cdot \sum_{t=1}^T \sum_{s=1}^S \left[
            \PR(\theta_{t,s}^+)
            - \PR(\hat{\theta}_t^+)
            + \PR(\theta_{t,s}^-)
            - \PR(\hat{\theta}_t^-)
        \right] \\
    &= O\left(\sqrt{N}\right)
        + (N_\KL S + 1) \cdot \calR_T(\minimizePR, \PR)
        + N_\KL \cdot \sum_{t=1}^T \sum_{s=1}^S \left[
            \PR(\theta_{t,s}^+)
            - \PR(\hat{\theta}_t^+)
            + \PR(\theta_{t,s}^-)
            - \PR(\hat{\theta}_t^-)
        \right] \\
    &= O\left(\sqrt{N}\right)
        + (N_\KL S + 1) \cdot \calR_T(\minimizePR, \PR)
        + N_\KL \cdot \sum_{t=1}^T \sum_{s=1}^S \left[
            \PRdagger(\varphi(\theta_{t,s}^+))
            - \PR(\hat{\theta}_t^+)
            + \PRdagger(\varphi(\theta_{t,s}^-))
            - \PR(\hat{\theta}_t^-)
        \right] \\
    &= O\left(\sqrt{N}\right)
        + (N_\KL S + 1) \cdot \calR_T(\minimizePR, \PR) \\
        &\qquad \qquad
        + N_\KL \cdot \underbrace{
                \sum_{t=1}^T \sum_{s=1}^S \left[
                    \PRdagger(\varphi(\theta_{t,s}^+))
                    - \PRdagger(\phi_t^+)
                \right]
            }_{\text{(I)}}
            + N_\KL \cdot \underbrace{
                \sum_{t=1}^T \sum_{s=1}^S \left[
                    \PRdagger(\phi_t^+)
                    - \PR(\hat{\theta}_t^+)
                \right]
            }_{\text{(II)}} \\
        &\qquad \qquad
        + N_\KL \cdot \underbrace{
                \sum_{t=1}^T \sum_{s=1}^S \left[
                    \PRdagger(\varphi(\theta_{t,s}^-))
                    - \PRdagger(\phi_t^-)
                \right]
            }_{\text{(III)}}
            + N_\KL \cdot \underbrace{
                \sum_{t=1}^T \sum_{s=1}^S \left[
                    \PRdagger(\phi_t^-)
                    - \PR(\hat{\theta}_t^-)
                \right]
            }_{\text{(IV)}}
\end{align*}

Term (I) is:
\begin{align*}
    \sum_{t=1}^T \sum_{s=1}^S \left[
            \PRdagger(\varphi(\theta_{t,s}^+))
            - \PRdagger(\phi_t^+)
        \right]
    &\leq L^\dagger \cdot
        \sum_{t=1}^T \sum_{s=1}^S
            \left\|\varphi(\theta_{t,s}^+) - \phi_t^+\right\|
        \tag{Lipschitzness of $\PRdagger$} \\
    &\leq  L^\dagger \cdot
        \sum_{t=1}^T \sqrt{ S
            \sum_{s=1}^S
            \left(\left\|\varphi(\theta_{t,s}^+) - \phi_t^+\right\|^2 \right)} \tag{Cauchy-Schwarz}
            \\
    &= L^\dagger T \sqrt{S \sum_{s=1}^S L_{\theta}^2 {\KL(\phi_t^+ || \varphi(\theta_{t,s}^+))}} \tag{\Cref{lemma:lip-phi-in-KL}}
     \\
    &\leq L^\dagger L_\theta T \cdot \sqrt{S \cdot \calR_S(\learnmodel, \KL) }
\end{align*}

and term (III) is analogous. Term (II) is
\begin{align*}
    \sum_{t=1}^T \sum_{s=1}^S \left[
            \PRdagger(\phi_t^+)
            - \PR(\hat{\theta}_t^+)
        \right]
    &= S \cdot \sum_{t=1}^T \left[
            \PRdagger(\phi_t^+)
            - \PRdagger(\varphi(\hat{\theta}_t^+))
        \right] \\
    &\leq L^\dagger S \cdot \sum_{t=1}^T
        \left\|\phi_t^+ - \varphi(\hat{\theta}_t^+)\right\| \tag{Lipschitzness of $\PRdagger$} \\
    &\leq L^\dagger S \cdot \sum_{t=1}^T L_\theta \sqrt{\KL(\phi_t^+ ||\varphi(\hat{\theta}_t^+))}\tag{\Cref{lemma:lip-phi-in-KL}}\\
    &\leq L^\dagger L_\theta \cdot  S \cdot \sum_{t=1}^T 
            \sqrt{
            \frac{1}{S} \sum_{s=1}^S
            \KL(\phi_t^+||\varphi(\theta_{t,s}^+))
            }
        \tag{$\hat{\theta}_t^+ := \frac{1}{S}\sum_{s=1}^S \theta_{t,s}^+$, convexity of $\KL(\phi^+_t||\varphi(\theta))$} \\
    &\leq L^\dagger  L_\phi T S \sqrt{\frac{1}{S}\calR_S(\learnmodel,\KL) }\\
    &= L^\dagger L_\phi T\sqrt{S\cdot\calR_S(\learnmodel,\KL)}
\end{align*}

and term (IV) is analogous. In total we have
\begin{align*}
    & \calR_N(\minimizePR, \PR) \\
    &= O\left(
            \sqrt{N}
            + N_\KL T \cdot \sqrt{S \cdot \calR_S(\learnmodel, \KL)}
            + (N_\KL S + 1) \cdot \calR_T(\minimizePR, \PR)
        \right) \\
    &= N \cdot O\left(
            \frac{1}{\sqrt{N}}
            + \sqrt{\frac{\calR_S(\learnmodel, \KL)}{S}}
            + \frac{\calR_T(\minimizePR, \PR)}{T}
        \right)
        \tag{$n = 2(N_\KL 2S + 1)T$} \\
    &= N \cdot O\left(
            \frac{1}{\sqrt{N}}
            + \sqrt{
                    \sqrt{\frac{d_\Theta \log\frac{1}{p'}}{S}}
                    + \sqrt{\epsilon_\KL d_\Theta}
                }
            + \sqrt{\frac{d_\Phi \log\frac{1}{p''}}{T}}
            + \sqrt{\epsilon_\LM d_\Phi}
        \right)
        \tag{by \Cref{theorem:regret-bound-for-indirectly-convex-functions},\Cref{theorem:regret-bound-for-learnmodel}, w.p. to be analyzed later} \\
    &= N \cdot O\left(
            \left(\frac{d_\Theta}{S}\log\frac{1}{p'}\right)^{1/4}
            + (\epsilon_\KL d_\Theta)^{1/4}
            + \left(\frac{d_\Phi}{T}\log\frac{1}{p''}\right)^{1/2}
            + (\epsilon_\LM d_\Phi)^{1/2}
        \right)
        \tag{for $a,b \geq 0$, $\sqrt{a+b} \leq \sqrt{a} + \sqrt{b}$; $\frac{1}{\sqrt{n}} \leq \sqrt{\frac{d_\Phi}{T}}$} \\
    &\leq N \cdot \left(1 + \left(\log\frac{1}{p'}\right)^{1/4} + \left(\log\frac{1}{p''}\right)^{1/2}\right)
        \cdot O\left(
            \left(\frac{d_\Theta}{S}\right)^{1/4}
            + (\epsilon_\KL d_\Theta)^{1/4}
            + \left(\frac{d_\Phi}{T}\right)^{1/2}
            + (\epsilon_\LM d_\Phi)^{1/2}
        \right) \\
    &= N \cdot \left(1 + \left(\log\frac{1}{p'}\right)^{1/4} + \left(\log\frac{1}{p''}\right)^{1/2}\right)
        \cdot O\left(
            \left(\frac{d_\Theta}{S}\right)^{1/4}
            + (\epsilon_\KL d_\Theta)^{1/4}
            + \left(\frac{d_\Phi N_\KL S}{N}\right)^{1/2}
            + (\epsilon_\LM d_\Phi)^{1/2}
        \right)
        \tag{$T = \frac{N}{N_\KL S + 1}$}
\end{align*}

Choose $\epsilon_\LM = \left(\frac{N_\KL}{N}\right)^{1/3}$ and $\epsilon_\KL = \frac{1}{4d_\Theta}\left(\frac{N_\KL}{N}\right)^{2/3}$. 

To balance the terms, set the number of steps for the outer algorithm to be $T = \frac{d_\Phi}{(\epsilon - \sqrt{\epsilon_\LM d_\Phi})^2}$, and the number of steps in $\learnmodel$ to be
\begin{align*}
    S
    = \frac{d_\Theta}
        {\left(
            \epsilon_\LM
            - \sqrt{\epsilon_\KL d_\Theta}
        \right)^2}
    = 4d_\Theta \left(\frac{N}{N_\KL}\right)^{2/3}
\end{align*}

Plugging these expressions for $\epsilon_\KL$, $\epsilon_\LM$, and $S$ in above, we have
\begin{align*}
    \calR_n(\minimizePR, \PR)
    &= N \cdot \left(1 + \left(\log\frac{1}{p'}\right)^{1/4} + \left(\log\frac{1}{p''}\right)^{1/2}\right)
        \cdot O\left(
            (d_\Theta d_\Phi)^{1/2}
            \left(\frac{N_\KL}{N}\right)^{1/6}
        \right) \\
    &= O\left(
            \left(1 + \left(\log\frac{1}{p'}\right)^{1/4} + \left(\log\frac{1}{p''}\right)^{1/2}\right)
            (d_\Theta + d_\Phi)
            N_\KL^{1/6}
            N^{5/6}
        \right)
\end{align*}

We would like to ensure that this bound holds with probability $p>0$. To that end, observe that the probabilistic terms are the high-probability bounds on $\calR_S(\learnmodel,\KL)$ and $\calR_T(\minimizePR,\PR)$. By recalling  \Cref{theorem:regret-bound-for-indirectly-convex-functions} and \Cref{theorem:regret-bound-for-learnmodel}, the probability that any of these bounds fails is at most
\begin{align*}
    p' + Tp_\LM
    = p' + T(p'' + S p_\KL)
    = p' + Tp'' + STp_\KL
\end{align*}
for any $p', p'' > 0$. For a choice of $p' = p/3$, $p'' = p/3T$, and $p_\KL = \frac{p N_\KL}{3n}$, this is at most $p$ as required. Finally, plugging these choices into the above regret bound yields
\begin{align*}
    \calR_n(\minimizePR, \PR)
    &= O\left(
            \left(1 + \left(\log\frac{1}{p'}\right)^{1/4} + \left(\log\frac{1}{p''}\right)^{1/2}\right)
            (d_\Theta + d_\Phi)
            N_\KL^{1/6}
            N^{5/6}
        \right) \\
    &= O\left(
            \left(1 + \left(\log\frac{1}{p}\right)^{1/4} + \left(\log\frac{T}{p}\right)^{1/2}\right)
            (d_\Theta + d_\Phi)
            N_\KL^{1/6}
            N^{5/6}
        \right) \\
    &= O\left(
            \left(1 + \sqrt{\log\frac{1}{p}}\right)
            (d_\Theta + d_\Phi)
            N_\KL^{1/6}
            N^{5/6}
            \sqrt{\log N}
        \right)
        \tag{$T \leq N$}
\end{align*}
with probability at most $p$ as required.
\end{proof}

\section{Additional Related Work}
\label{sec:additional-related-work}
In this section, we provide additional related work in performative prediction and a detailed comparison of our work and some closely related work.

Performative prediction is a new type of supervised learning problem in which the underlying data distribution shifts in response to the deployed model \cite{perdomo2020performative,brown2020performative, drusvyatskiy2020stochastic,izzo2021learn,li2022state,maheshwari2022zeroth,ray2022decision, mofakhami2023performative}. It is also called the \emph{decision-dependent risk minimization} problem \cite{maheshwari2022zeroth, li2022multi, yuan2023learning}. In particular, \citet{perdomo2020performative} first propose the notion of the \emph{performative risk} defined as 
$\PR(\theta):= \E_{z \sim \calD(\theta)}[\ell(\theta;z)]$
where $\theta$ is the model parameter, and $\calD(\theta)$ is the induced distribution due to the deployment of $\theta$. 

One of the major focuses of performative prediction is to find the optimal model $\theta_{\textsf{OPT}}$ which achieves the minimum performative prediction risk:
$\theta_{\textsf{OPT}}:= \argmin_{\theta\in \Theta} \PR(\theta)$,
or performative stable model $\theta_{\textsf{ST}}$, which is optimal under its own induced distribution:
$\theta_{\textsf{ST}}:= \argmin_{\theta\in \Theta}\E_{z \sim \calD(\theta_{\textsf{ST}})}[\ell(\theta;z)]$
. In particular, one way to find a performative stable model $\theta_{\textsf{ST}}$ is to perform repeated retraining \cite{perdomo2020performative}. 

In order to get meaningful theoretical guarantees on any proposed algorithms, works in this field generally require particular assumptions on the mapping between the model parameter and its induced distribution (e.g., the smoothness of the mapping), 
or require multiple rounds of deployments and observing the corresponding induced distributions, which can be costly in practice \cite{jagadeesan2022regret, mendler2020stochastic}. A few recent works are on finding performative optimal solutions without explicitly making the convexity assumption. For example, \citet{dong2021approximate} does not explicitly convexity assumption, but they focus on optimization heuristics that are not guaranteed to minimize performative regret. below, we will provide the discussions for three of them. 

 In addition, minimizing the performative risk often requires knowing a specific model for the distribution map $\calD(\cdot)$ that can be fit. To ensure performative risk minimization is tractable, one also requires imposing structural assumptions on the distribution map. For example,  \citeauthor{izzo2021learn} makes parametric assumptions on $\calD(\theta)$ and assumes that $\calD(\theta)$ has a continuously differentiable density $p(z; \varphi(\theta))$, where $\varphi(\cdot): \Theta \rightarrow \Phi$ represents the mapping from the model parameter space $\Theta$ to the data distribution parameter space $\Theta$. \citet{miller2021outside} assume the underlying data distribution follows a location family distribution, and then impose a \emph{mixture dominance} assumption on the distribution map $\calD(\cdot)$ from which it follows that $\PR(\theta)$ is convex; this again leads to a gradient-based optimization algorithm. Similar work include \cite{mendler2020stochastic, izzo2021learn, drusvyatskiy2020stochastic,cutler2021stochastic}, to name a few. 

\paragraph{Comparison with \citet{miller2021outside}}

\citet{miller2021outside} identifies \emph{mixture dominance condition} for any particular model parameter pairs under which the performative risk is convexity. In particular, they posit a simple distribution map in which $\varphi(\theta) = \phi_0 + M\theta$, where $M \in \R^{d_\Phi \times d_\Theta}$ and $\phi_0 \in \Phi$ is some ``base'' distribution parameter; in other words, they assume that the data population reacts to a model by shifting each of their features according to some linear transformation of the model parameter. Their algorithm for this special case works in two stages: first estimating $\varphi_0$ and $M$ by deploying random models; then, once this distribution map has been accurately estimated, the performative loss is convex in $\theta$, and can be optimized offline. The distribution map estimation takes $O(d_\Theta/\epsilon)$ samples to obtain an $\epsilon$-suboptimal model.

\paragraph{Comparison with \citet{jagadeesan2022regret}}
For example, closely related is a recent paper that proposes using the Lipschitz bandit approach to solve the performative prediction problem \cite{jagadeesan2022regret}. The major differences between this work and their work are: first, we define the regret w.r.t $N$ rather than w.r.t $T$, which is a more realistic measure in the performative prediction setting; second, their regret has exponential dependency on the ``zooming dimension'' $d$ (which is roughly the model parameter $\modelDimen$), while our dependency on the model and distribution dimensions are both linear. 

\paragraph{Comparison with \citet{maheshwari2022zeroth}} Another closely related work is \citet{maheshwari2022zeroth} uses zeroth-order methods for the convex-concave minimax problem. Specifically, they proposed to formulate the performative prediction problem as the \emph{Wasserstein distributionally robust learning with decision-dependent data} problem, and further reduce it to a constrained finite-dimensional smooth convex-concave min-max problem, and propose a zeroth-order random reshuffling-based algorithm to solve the problem without assuming any other structure on the curvature of the min-max loss. Similar to ours, they also use the zeroth-order method to perform their optimization procedure; different from ours, they approach the performative prediction problem through the angle of robustness, which accounts for model misspecification in their analysis.

\paragraph{Other Aspects of performative prediction}  

Also related are the recently developed lines of work on the \emph{multiplayer} version of the performative prediction problem \cite{piliouras2022multi, narang2022multiplayer, li2022multi, foster2023complexity}. While existing strategic classification and performative prediction problems focus primarily on the interplay between a single learner and the population that reacts to the learner's actions, this line of work takes into account competition from multi-learners, and develop performatively stable equilibria and Nash equilibria of the game. Similarly, \cite{yuan2023learning} confront multiple interactive models in some dynamic environments. Another line of work is the economic aspects of performative prediction \cite{hardt2022performative, mendleranticipating}. From the optimization aspect, \citet{wood2022online} focuses on the optimization aspect of finding the performative optimal point and offers an online stochastic primal-dual algorithm for tracking equilibrium trajectories. 
Also related is the recent development of the concept called induced domain adaptation \cite{chen2023model}, whose primary focus is to study the \emph{transferability} of a particular model trained primarily on the source distribution and provide theoretical bounds on its performance on its induced distribution, which is helpful in estimating the effect of a given classifier when repeated retraining is unavailable.

\paragraph{Theoretical comparisons to some existing methods in the convex case}

When the problem reduces to a convex, differentiable Lipschitz case, our algorithm will be reduced to the convex case provided in \cref{sec: warmup} (the warm-up setting), which achieves a $\tilde{O}(\sqrt{d N})$ regret bound. This implies that our algorithm achieves a $\Delta$-suboptimal model with $O(d/\Delta^2)$ samples (see \cref{lemma:regret-bound-for-model-convex-performative-risk}).

Here, we compare the three papers that the reviewer mentioned:
\begin{itemize}
    \item \citet{izzo2021learn} focus on a single-distribution Gaussian distribution with a fixed variance setting while we cover a boarder range of settings. Their theoretical guarantee shows that the proposed method converges to a performative optimal point as the number of iterations $T\approx \sigma^{-4/5}$ where $\sigma$ bounds the output of PR from the PR of the optimal performative point. The sample required at each iteration is $O(1/\sigma^2\log T)$.
    \item \citet{miller2021outside} show that when the distribution maps $\mathcal D(\cdot)$ form a location-scale family and when the model dimension is $O(d)$, computing a $\Delta$-suboptimal classifier requires $O(d/\delta)$ samples. We do not require these assumptions.
    \item  \citet{perdomo2020performative} focus on achieving a performative stable point while we focus on attaining a performative optimal point.
\end{itemize}

\section{Plots For Empirical Results}
\label{sec:empirical}

\begin{figure}[h!]
    \centering
    \begin{subfigure}
        \centering
        \includegraphics[width=0.6\textwidth]{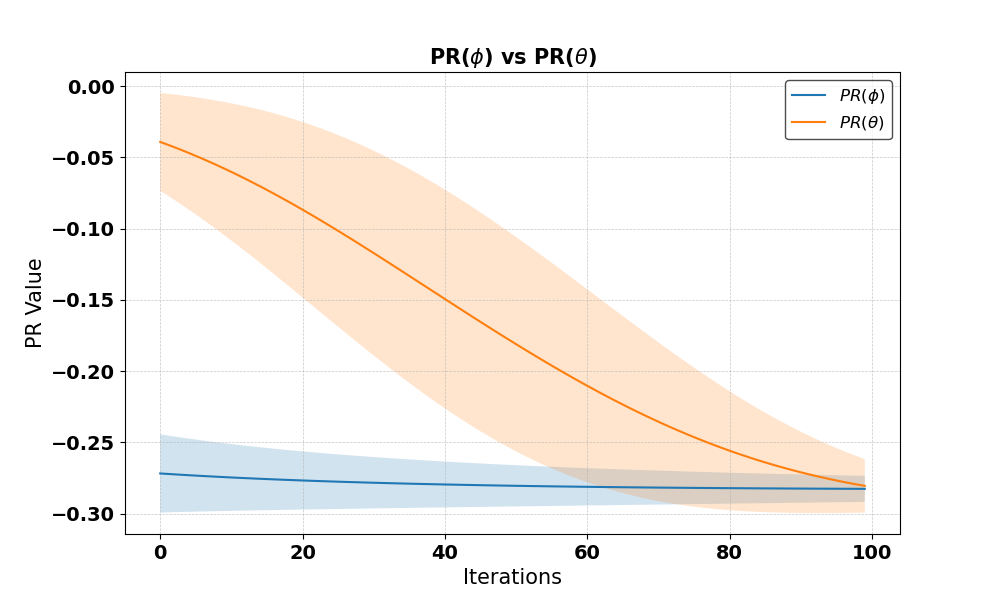}
        \caption{$\theta\in [0.2, 0.4]$}
    \end{subfigure}
    \begin{subfigure}
        \centering
         \includegraphics[width=0.6\textwidth]{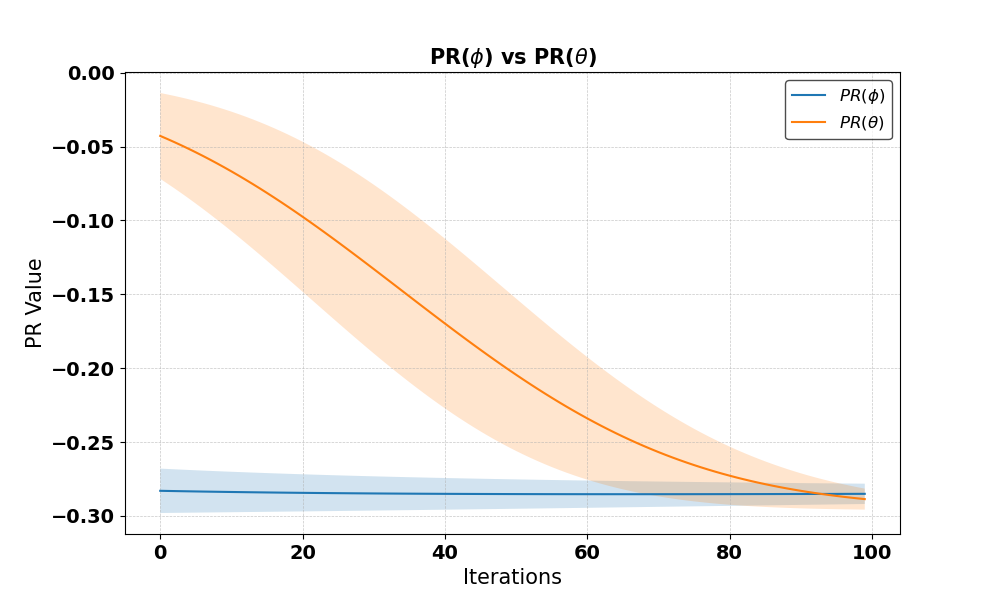}
        \caption{ $\theta\in [0.4, 0.8]$}
    \end{subfigure}
    \begin{subfigure}
        \centering
        \includegraphics[width=0.6\textwidth]{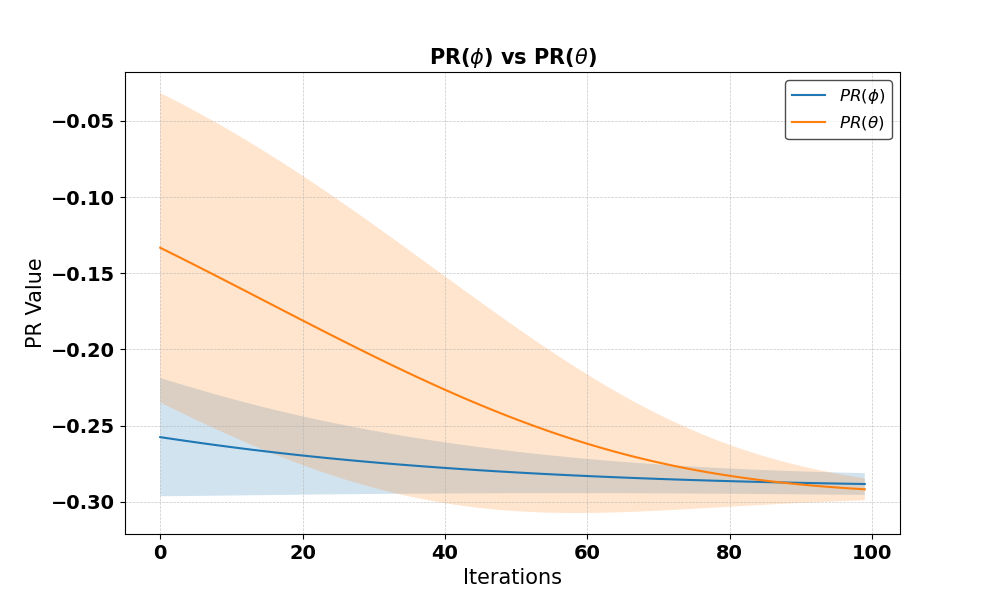}
        \caption{ $\theta\in [0.4, 0.8]$}
    \end{subfigure}
    \caption{Empirical results comparing baseline method (zeroth-order optimization without reparametrization, orange curve) vs. our
method (zeroth order optimization after reparametrization, (blue curve) based on \cref{example:mixture-dominance-too-strong-condition}.}
    \label{fig:lr-social-cost-diff}
\end{figure}


\end{document}